\documentclass[12pt, draftclsnofoot, onecolumn]{IEEEtran}
\usepackage{graphicx}
\usepackage{amsthm}
\usepackage{epsfig}
\usepackage{latexsym}
\usepackage{amsfonts}
\usepackage{here}
\usepackage{rawfonts}
\usepackage[latin1]{inputenc}
\usepackage[T1]{fontenc}
\usepackage{calc}
\usepackage{capitalgreekitalic}
\usepackage{url}
\usepackage{enumerate}
\usepackage{color}
\usepackage[tbtags]{amsmath}
\usepackage{amssymb}
\usepackage{upref}
\usepackage{epic,eepic}
\usepackage{times}
\usepackage{dsfont}
\usepackage{comment}
\usepackage{cite}
\usepackage{bbm} 
\usepackage{optidef}
\usepackage{empheq}
\newtheorem{proposition}{\bf Proposition}

\usepackage{dsfont}
\newcounter{step}
\newlength{\totlinewidth}
\newenvironment{algorithm}{%
  \rule{\linewidth}{1pt}
  \begin{list}{}%
    {\usecounter{step}%
      \settowidth{\labelwidth}{\textbf{Step 2:}}%
      \setlength{\leftmargin}{\labelwidth}%
      \setlength{\topsep}{-2pt}%
      \addtolength{\leftmargin}{\labelsep}%
      \addtolength{\leftmargin}{2mm}%
      \setlength{\rightmargin}{2mm}%
      \setlength{\totlinewidth}{\linewidth}%
      \addtolength{\totlinewidth}{\leftmargin}%
      \addtolength{\totlinewidth}{\rightmargin}%
      \setlength{\parsep}{0mm}%
      \raggedright}}%
  {\end{list}%
  \rule{\linewidth}{1pt}}
\newcounter{substep}

  {\end{list}}

\newlength{\aligntop}
\newlength{\alignbot}
 \makeatletter
\renewenvironment{align}{%
  \vspace{\aligntop}
  \start@align\@ne\st@rredfalse\m@ne
}{%
  \math@cr \black@\totwidth@
  \egroup
  \ifingather@
    \restorealignstate@
    \egroup
    \nonumber
    \ifnum0=`{\fi\iffalse}\fi
  \else
    $$%
  \fi
  \ignorespacesafterend%
  \vspace{\alignbot}\par\noindent
} \makeatother

\IEEEoverridecommandlockouts

\usepackage{algorithm}%,algpseudocode}
\usepackage{algorithmic}
\usepackage{subfigure}
\begin{document}
%\clearpage
\title{\huge Distributed Multi-agent Meta Learning for Trajectory Design in Wireless Drone Networks}%Competitive Market for Joint Access and Backhaul Resource Allocation in Satellite-Drone Networks\vspace{-0.26cm}}

\author{
\IEEEauthorblockN{Ye Hu,  {Mingzhe Chen}, \emph{Member, IEEE}, Walid Saad, \emph{Fellow, IEEE}, H. Vincent Poor, \emph{Life Fellow, IEEE}, and Shuguang Cui, \emph{Fellow, IEEE}\vspace{0.2cm}
\thanks{{A preliminary version of this work will be presented at the IEEE Global Communications Conference \cite{hu2020meta}.}}}
\thanks{Y. Hu, W. Saad are with the Wireless@VT, Bradley Department of Electrical and Computer Engineering, Virginia Tech, Blacksburg, VA, USA, 24061, Emails: \protect{yeh17@vt.edu}, \protect{walids@vt.edu}.}
\thanks{M. Chen is with the Department of Electrical Engineering, Princeton University, Princeton, NJ, USA, 08544, and the Future Network of Intelligence Institute, Chinese University of Hong Kong, Shenzhen, China, 518172, Email: \protect{mingzhec@princeton.edu}.}
\thanks{H. V. Poor is with the Department of Electrical Engineering, Princeton University, Princeton, NJ, USA, 08544, Emails: \protect{poor@princeton.edu}.}
\thanks{S. Cui is currently with the Shenzhen Research Institute of Big Data and Future Network of Intelligence Institute (FNii), the Chinese University of Hong Kong, Shenzhen, China, 518172, Email:  \protect{e-mail: shuguangcui@cuhk.edu.cn}.}}

%Ye Hu\IEEEauthorrefmark{1},  Mingzhe Chen\IEEEauthorrefmark{2}\IEEEauthorrefmark{3}, Walid Saad\IEEEauthorrefmark{1}, H. Vincent Poor\IEEEauthorrefmark{2}, and Shuguang Cui\IEEEauthorrefmark{3}\vspace{0.2cm}}
  
%\IEEEauthorblockA{\IEEEauthorrefmark{1}\small Wireless@VT, Bradley Department of Electrical and Computer Engineering, Virginia Tech, Blacksburg, VA, USA.\\
%} 
%\IEEEauthorblockA{\IEEEauthorrefmark{2}\small Department of Electrical Engineering, Princeton University, Princeton, NJ, USA.\\
  %The Future Network of Intelligence Institute, Chinese University of Hong Kong, Shenzhen, China. Email: \protect{mingzhec@princeton.edu}.}
 % \IEEEauthorblockA{\IEEEauthorrefmark{2}\small Department of Electrical Engineering, Princeton University, Princeton, NJ, USA.\\
%} 
%\IEEEauthorblockA{\IEEEauthorrefmark{3}\small The Future Network of Intelligence Institute, Chinese University of Hong Kong, Shenzhen, China.\\
%Emails: \{yeh17,walids\}@vt.edu, \{mingzhec, poor\}@princeton.edu, robert.cui@gmail.com.} 
%%\IEEEauthorblockA{\IEEEauthorrefmark{3}\small Tsinghua SEM Advanced ICT Lab, Beijing, China 100084. Email: haom@sem.tsinghua.edu.cn.}
\vspace{-0.8cm}



\maketitle
\vspace{-1.5cm} 
\begin{abstract}
\boldmath
 {In this paper, the problem of the trajectory design for a group of energy-constrained drones operating in dynamic wireless network environments is studied. In the considered model, a team of drone base stations (DBSs) is dispatched to cooperatively serve clusters of ground users that have dynamic and unpredictable uplink access demands.  In this scenario, the DBSs must cooperatively navigate in the considered area to maximize coverage of the dynamic requests of the ground users. This trajectory design problem is posed as an optimization framework whose goal is to find optimal trajectories that maximize the fraction of users served by all DBSs. 
 To find an optimal solution for this non-convex optimization problem under unpredictable environments, a value decomposition based reinforcement learning (VD-RL) solution coupled with a meta-training mechanism is proposed. This algorithm allows the DBSs to dynamically learn their trajectories while generalizing their learning to unseen environments.
Analytical results show that, the proposed VD-RL algorithm is guaranteed to converge to a local optimal solution of the non-convex optimization problem. Simulation results show that, even without meta-training, the proposed VD-RL algorithm can achieve a $53.2\%$ improvement of the service coverage and a $30.6\%$ improvement terms of the convergence speed, compared to baseline multi-agent algorithms. Meanwhile, the use of meta-learning improves the convergence speed of the VD-RL algorithm by up to $53.8\%$ when the DBSs must deal with a previously unseen task.
}
 
 \end{abstract}

\begin{IEEEkeywords} 
Drones, network optimization, multi-agent reinforcement learning, meta-learning.
\end{IEEEkeywords}
%{\small \emph{ {Index Terms}---  {Drones, trajectory design, multi-agent reinforcement learning, meta-learning.}}}
%\newpage


\renewcommand{\thefootnote}{\arabic{footnote}}

%\footnotetext[1]{This research was supported in part by the National Research Foundation for the Doctoral Program of Higher Education of China (No.20120005110007), the NSFC (No. 61271257), and by the U.S. National Science Foundation under Grant 1513697.}
%\thanks{This research was supported by the U.S. National Science Foundation under Grant 1513697.}

% keywords

%\begin{keywords}
%uplink-downlink decoupling, LTE-unlicensed, user association, learning, echo state networks;        
%\end{keywords}
%\renewcommand{\thefootnote}{\fnsymbol{footnote}}

%\footnotetext[1]{}
%This work is supported by the National Natural Science Foundation of China (61271177)and the Fundamental Research Funds for the Central Universities.
% \fi
\IEEEpeerreviewmaketitle
 \vspace{-0.35cm}  

\section{Introduction}

{Aerial wireless communication platforms carried by drones can provide a cost-effective, flexible approach to boost the coverage and capacity of future wireless networks \cite{8755300, 6214709,8533634}. However, effectively deploying a group of drone base stations (DBSs) for providing timely on demand wireless connectivity to ground users in dynamic wireless environments is still an important open problem. In particular, designing trajectories for a group of independent DBSs is challenging particularly when the DBSs only have limited information on the wireless requests of the ground users which are often highly unpredictable and dynamic.

\vspace{-0.35cm}
\subsection{Related Works}
 
The existing literature in \cite{7888557, 8648498, 8247211, 7389838, 9013759, huang2019reinforcement, 8432464, 8654727, 9154432, 8807386, 8727504, foerster2017counterfactual, rashid2018qmix} studied a number of problems related to trajectory design for drone-based wireless networks. 
The work in \cite{7888557} studies the drone trajectory optimization problem by jointly considering both the drone's communication throughput and its energy consumption.
 The authors in \cite{8648498} design the trajectory of a a solar-powered DBS to enhance its wireless communication performance. In \cite{8247211}, the problem of trajectory design and user association in a multi-drone communication system is solved with a block coordinate descent solution. The authors in \cite{7389838} propose a dynamic trajectory control algorithm to improve the communication performance of the DBSs. The authors in \cite{9013759} optimize time allocation, reflection coefficient adjustment, and DBS trajectory in backscatter communication networks.
  Despite their promising results, these existing works \cite{7888557, 8648498, 8247211, 7389838, 9013759} do not consider practical DBS-assisted wireless networks in which the ground user requests for wireless service follow unpredictable patterns. Indeed, the optimization based solutions in \cite{7888557, 8648498, 8247211, 7389838, 9013759} are not suitable to design DBS trajectories when the user requests are unknown and unforeseeable.   
  %any trivial changes on the user request pattern could has a direct impact on the effectiveness on the DBSs' trajectory design. Hence, a solution that can swiftly adapt the DBSs in various unknown environment should be provided. 

 Recently, there has been significant interest in using machine learning for drone trajectory optimization in unpredictable environments \cite{huang2019reinforcement, 8432464, 8654727, 9154432, 8807386, 8727504, foerster2017counterfactual, rashid2018qmix}. In \cite{huang2019reinforcement}, the problem of DBS trajectory optimization in an uplink non-orthogonal multiple-access (NOMA) system is studied and solved using reinforcement learning (RL).  The work in \cite{8432464} solves the problem of energy-efficient drone trajectory design using deep RL. The authors in \cite{8654727} study the problem of  interference-aware path planning for a group of DBSs using machine learning. The work in \cite{9154432} uses deep RL to design DBS trajectories so as to minimize the age of information (AoI) for sensing tasks in a multi-drone network.
 Meanwhile, in \cite{8807386}, the authors develop a multi-agent reinforcement learning (MARL) framework for dynamic resource allocation in wireless networks with multiple drones.
 The authors in \cite{8727504} propose a new approach to predict the users' mobility patterns and, then, they use a multi-agent Q-learning algorithm to design trajectories for a group of DBSs. The works in \cite{foerster2017counterfactual} and \cite{rashid2018qmix} propose distributed multi-agent algorithms that allow a group of agents to update their individual strategies considering the team benefits. 
However, most of the existing MARL solutions such as those in \cite{8654727}, and \cite{8727504, foerster2017counterfactual, rashid2018qmix} require DBSs to share their states and actions while searching for the optimal strategies. These traditional RL solutions have high complexity, as they solve multi-agent problems by updating strategies based on the entire set of agents' actions and strategies whose dimension increases exponentially with the number of agents. 
Meanwhile, the MARL solutions in \cite{9154432} and \cite{8807386} allow the agents to search strategies independently based on their own actions and states. However, using these RL solutions, the DBSs cannot optimize the sum utilities of all DBSs, and thus, cannot maximize the overall coverage of the ground users, since the DBSs are optimizing their individual utilities. 
In addition, traditional RL solutions such as those used in  \cite{8654727, 9154432, 8807386, 8727504, foerster2017counterfactual, rashid2018qmix} cannot efficiently adapt the trajectories of the DBSs to unseen environments, as they are often overfitted to their training tasks. This is because the hyper-parameters, exploration strategies, and initializations of traditional RL algorithms are manually adjusted for fitting the training tasks.  Once the agent faces an unseen task, manually adjusted RL algorithms may not converge to the optimal solution and, even if they do, the convergence speed will be very slow.
As a result, the traditional RL algorithms in \cite{huang2019reinforcement, 8432464, 8654727, 9154432, 8807386, 8727504, foerster2017counterfactual, rashid2018qmix} cannot effectively find optimal DBS trajectories in unseen environments. Finally, we note that in \cite{hu2020meta}, we studied the problem of trajectory design for a single DBS operating in a dynamic environment  using meta-learning. However, this prior work relies on a simple algorithm that cannot be scaled to larger networks.
%Meta-learning, also known as ``learning to learn'', provides a solution to adapt the DBSs to unpredictable, unseen environments based only on a small number of former experiences.

%that is, it can reaches to the optimal trajectory designs in unseen environments with only a few training examples.
%In \cite{hu2020meta}, we proposed a meta-learning based solution that tunes a RL algorithm that directs a DBS for best coverage to the unpredictable, dynamic user requests.  However, in \cite{hu2020meta}, only one DBS is deployed, which results with an inadequate coverage. Thus, there is a need to introduce new solutions that coordinate a group of DBSs in wireless environments.

\vspace{-0.35cm}
\subsection{Contributions}
The main contribution of this paper is a novel distributed framework for designing the trajectories of a group of cooperative DBSs in unpredictable, dynamic environments.  %In particular, the proposed approach uses meta-learning to enable a team of DBSs to navigate in unseen environments while providing on demand coverage to the ground user access requests. 
To our best knowledge, \emph{this is the first work that designs trajectories for a team of DBSs in unpredictable, dynamic environments using a multi-agent meta reinforcement learning solution.}
%Although some works such as \cite{park2019meta,8761319,8723067} have devoted themselves to solve problems in dynamic wireless environments with meta-learning algorithms, these works have not provided an solution that can effectively coordinate a group of agents in the dynamic environments. A novel meta-learning framework that is capable of designing trajectories for a group DBSs in unseen environments should be proposed for the considered problem. Unlike previous meta-based solutions \cite{park2019meta,8761319,8723067} who optimizes performance of a single agent with predictions, we propose a novel multi-agent meta reinforcement learning algorithm that optimizes a group of DBSs' performance in dynamic environments.
In brief, our key contributions include: 
\begin{itemize}
\item We consider a practical drone-aided wireless system in which a team of DBSs cooperatively navigate in an area, under strict energy constraints and limited information on surrounding environments, with the  goal of providing uplink wireless connectivity to ground users. The DBSs can provide on-demand coverage to the ground users while adapting their trajectories to those users' unpredictable access requests. We formulate this trajectory design problem as an optimization framework whose structure is shown to be non-convex. To solve this problem, the ``myopic'' DBSs, which have only limited access to the information of ground users, seek to find trajectories that maximize the expected portion of served users -- a wireless coverage metric that we use as the team utility of the group of DBSs.

\item To solve the formulated trajectory optimization problem, we propose a novel, distributed, value decomposition reinforcement learning (VD-RL) algorithm. %The VD-RL algorithm can learn how to decompose and attribute the team utility to each DBS thus enabling each DBS to update its strategy independently with its own states and actions. 
This algorithm is shown to reach a local optimal solution of the studied non-convex problem without requiring the DBSs to share their actions, states, or strategies. This makes the proposed VD-RL algorithm much less complex than traditional distributed MARL algorithms (e.g., those in \cite{8654727}, and \cite{8727504, foerster2017counterfactual, rashid2018qmix}), as the DBSs can update their strategies based on their own low-dimensional actions and states. 
The proposed VD-RL algorithm allows the DBSs to independently select strategies that maximize the team utility by decomposing and attributing this team utility to each DBS. Thus, with the proposed VD-RL algorithm, the DBSs can find a local optimal solution, which yields a much higher team utility, compared to the one that can be achieved by existing MARL solutions in \cite{9154432} and \cite{8807386}.

\item To improve the convergence speed of VD-RL for unseen environments, we propose a meta training mechanism that uses an optimization based solution to meta train the VD-RL algorithm. In particular, it seeks to find an optimal policy and value function initialization with a proper estimation on a distribution of environments, for the VD-RL algorithm. Using the this meta-learning approach, the VD-RL solution can quickly converge to a team optimal strategy when faced with an unseen task. The proposed optimization based meta training mechanism has a lower complexity compared to the meta training solutions in \cite{houthooft2018evolved} and \cite{ritter2018been}, because it does not require additional neural networks for meta-learning. Using the proposed approach, the DBSs can be meta-trained independently with their own actions and states. In particular, the proposed  meta-learning approach prepares the DBSs to cope with various tasks instead of a single sampled task as in \cite{xu2018meta}, by using a meta-trained initialization.  

%\item Simulation results show that proposed VD-RL algorithm can improve the service coverage and convergence speed by up to $53.2\%$ and $30.6\%$, compared to the traditional multi-agent algorithms. The results also show that the proposed meta-training method can direct the DBSs in an unseen environment a $5.6\%$ faster convergence speed, and $53.8\%$ higher coverage, compare to the vanilla VD-RL algorithm. %The heavy ball algorithm also yields $10\%$ improvement on convergence speed, compared to the sub-gradient algorithm.. 
\end{itemize}

Simulation results show that proposed VD-RL algorithm can improve the service coverage and convergence speed by up to $53.2\%$ and $30.6\%$, compared to traditional multi-agent algorithms. The results also show that using the proposed meta training mechanism, our VD-RL algorithm can find optimal trajectories in an unseen environment with a $53.8\%$ faster convergence speed, and a $5.6\%$ better coverage, compared to the VD-RL algorithm without meta-learning. 

The rest of this paper is organized as follows. The system model and problem formulation are described in Section \uppercase\expandafter{\romannumeral2}. In Section \uppercase\expandafter{\romannumeral3}, the proposed algorithm is developed and discussed. In Section \uppercase\expandafter{\romannumeral4}, numerical simulation results are presented. Finally, conclusions are drawn in Section \uppercase\expandafter{\romannumeral5}.}
 
\begin{figure}
  \centering
  \includegraphics[width=11 cm]{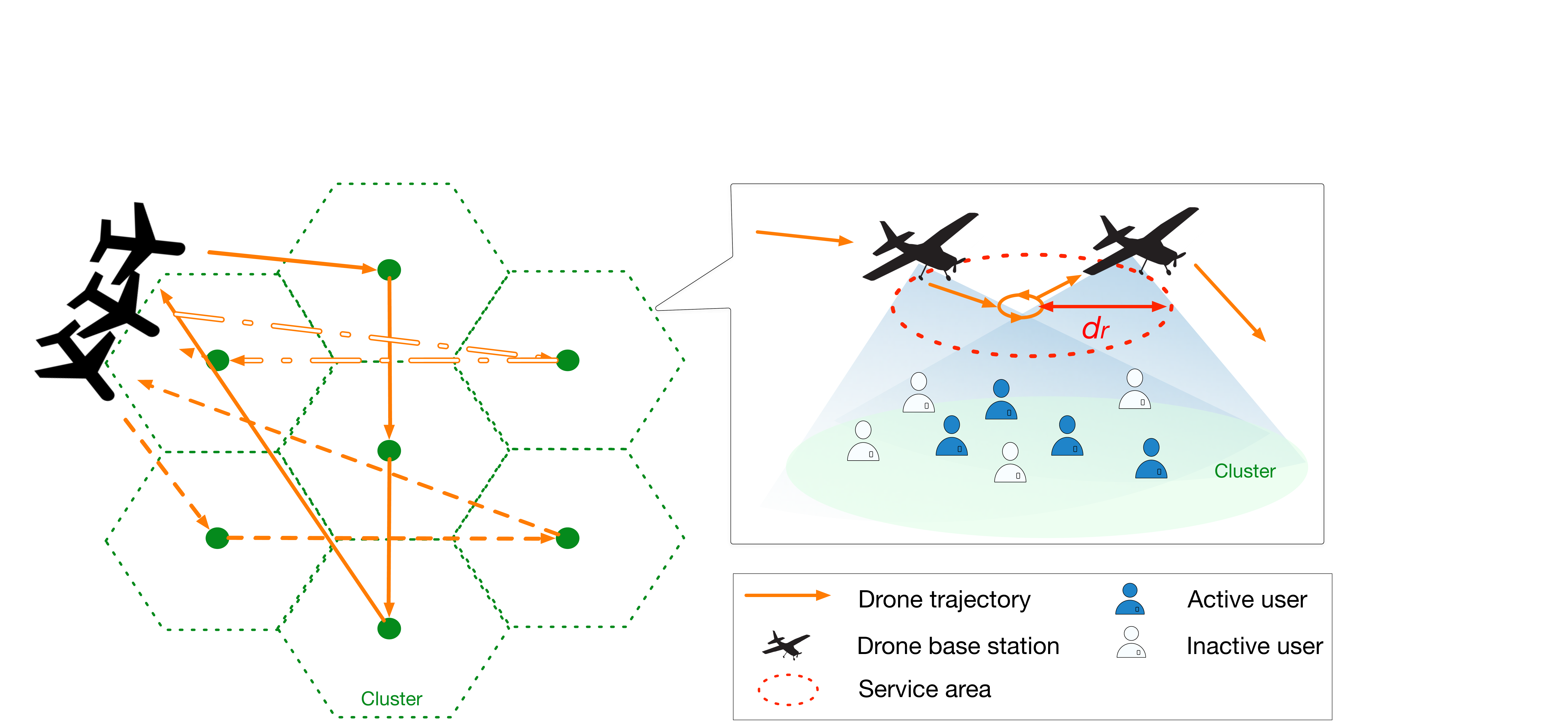}
  \caption{\footnotesize{Network topology.}}
  \label{Fig. 1}
  \centering
  \vspace{-0.8cm}
\end{figure}
\vspace{-0.4cm}
\section{System Model and Problem Formulation}
%\vspace{-0.12cm}
Consider a geographical area within which a set $\mathcal{U}$ of  $U$ randomly deployed terrestrial users request uplink data service. A set $\mathcal{N}$ of  $N$ fixed-wing DBSs are dispatched to satisfy the uplink access requests of those ground users, as shown in Fig. \ref{Fig. 1}. In the considered area, an user that requests data service is called an \emph{active user}, otherwise, it will be an \emph{inactive user}. We assume the users to be separated into different groups, each of which is called a \emph{cluster}. We also assume that, at any given time, each cluster will be served by a single DBS. The set of such clusters is denoted as $\mathcal{C}$. The DBSs will travel across the clusters in a steady straight-and-level flight (SLF), and hover over each cluster with a steady circular flight (SCF), at a constant speed $V_s$ \cite{7888557}. Each DBS $n\in\mathcal{N}$ flies at its own constant altitude $H_n$ to avoid collision with other DBSs. All the DBSs must return to their original location $O$ within a time period $T$ for battery charging. Moreover, all DBSs are assumed to have the same batter capacity. %considering its propulsion energy consumption during the process. 
The trajectory that records DBS $n$'s movement within the time period $T$ is represented by a vector $\boldsymbol{\xi}_{n}=\left[{\xi}_{n,1},{\xi}_{n,2},\ldots,{{\xi}_{n,K}}\right]^\top$, with ${{\xi}_{n,k}}\in\mathcal{C} \cup \left\{O\right\}$ being the $k$-th cluster that DBS $n$ serves, or the initial location $O$ that DBS $n$ returns to after serving the users, and $K$ being the maximum number of locations each DBS can fly across under its energy constraints. In other words,  each DBS $n$ flies across no more than $K$ different locations, each of which is called one step on the trajectory of DBS $n$, the set of such steps is denoted as $\mathcal{K}$. For example, $\boldsymbol{\xi}_{n}=\left[{\xi}_{n,1}, O, \ldots, O\right]$, with ${\xi}_{n,1}\in\mathcal{C}$, implies that DBS $n$ serves cluster ${\xi}_{n,1}$ and then flies back to the initial location $O$, while $\boldsymbol{\xi}_{n}=\left[{\xi}_{n,1}, {\xi}_{n,2}, \ldots, {\xi}_{n,K-1}, O\right]$, with ${\xi}_{n,k}\in\mathcal{C}$, implies that DBS $n$ serves clusters ${\xi}_{n,1}, {\xi}_{n,2}, \ldots, {\xi}_{n,K-1}$, and then flies back to the initial location $O$. %The set of DBSs' trajectories is $\boldsymbol{\xi}$, the set of all possible trajectories of the DBSs is given by $\mathcal{E}$. 

%Without loss of generality, we consider a three-dimensional (3D) Cartesian coordinate system such that each user $u\in\mathcal{U}$ locates at $\left(x_{u},y_{u},0\right)$. We also assume that the DBS, dispatched from $\left(0,0,H\right)$, will travel in a steady straight-and-level flight (SLF), and hover in a steady circular flight, with constant speed $V\left(\boldsymbol{S}_k\right)$ and a constant altitude $H$ \cite{}. Meanwhile, the DBS must return to $\left(0,0,H\right)$  within a time period $T$ to get charged, considering its propulsion energy consumption during the process. The velocity of the DBS at one time epoch $t$ is denoted as $\boldsymbol{v}_t=\left[v_{x,t},v_{y,t},0\right]$,  such that, for all $0\le t\le T$, $\left\| {\boldsymbol{v}_t} \right\|=\sqrt {v^2_{x,t}+v^2_{y,t}} =V$. The vector of DBS's changing velocity within the studied period is then denoted as $\boldsymbol{v}=\left[\boldsymbol{v}_1, \ldots, \boldsymbol{v}_T\right]$, the set of all possible velocity vectors for this DBS is $\mathcal{V}$. The trajectory tracing the location of DBS with velocity $\boldsymbol{v}$ is denoted as $\boldsymbol{j}\left(\boldsymbol{v}\right)=\left[,,H\right]$.

\vspace{-0.2cm}
\subsection{Communication Performance Analysis}
In our network, the users adopt an orthogonal frequency division multiple access (OFDMA) technique and transmit data over a set of uplink \emph{resource blocks} (RBs) \cite{6251827}. Each dispatched DBS will arbitrarily allocate one RB to each one of its associated users within a cluster. We assume that each DBS can keep serving its associated users within a $d_r$-meter radius over each cluster as shown in Fig.  \ref{Fig. 1}. The area within this range is called a \emph{service area}.
Also, user $u$ is assumed to request a total of $b _u$ data (in bits) at time epoch $t_u$. The DBSs must satisfy these requests within their battery capacitities.  Let $\boldsymbol{b}=\left[b_1,\ldots,b_U\right]$ and $\boldsymbol{t}=\left[t_1,\ldots,t_U\right]$ be, respectively, the vector of quantity and occurrence of the users' access request in the network. The quantity $b _u$ and active time $t_u$ are assumed to be independent random variables that follow unknown distributions. Hereinafter, we denote $\boldsymbol{z} = \left[\boldsymbol{b}, \boldsymbol{t}\right]$  as one \emph{realization} of the users' access request within duration $T$. In this model, the deployed DBSs are \emph{myopic}, that is, they know the access quantities and active time of only the users that they are currently serving. 
The path loss (in dB) of the line-of-sight (LoS) and non-line-of-sight (NLoS) air-to-ground communication links between DBS $n$ to user $u$ are given by the popular air-to-ground model in \cite{7037248}
 \begin{equation}\label{eq:pl}
\begin{split}
&h^{\textrm{LoS}}_{u,n} = 20\log\left(\frac{4\pi f_c d_{u,n}}{c}\right)+\varsigma^{\textrm{LoS}}_{u,n},\\
&h^{\textrm{NLoS}}_{u,n} = 20\log\left(\frac{4\pi f_c d_{u,n}}{c}\right)+\varsigma^{\textrm{NLoS}}_{u,n},
\end{split}
\end{equation}
where $f_c$ is the carrier frequency of the communication link between DBS $n$ and user $u$,  $d_{u, n}$ is the distance between user $u$ and DBS $n$, and $c$ is the speed of light. $\varsigma^{\textrm{LoS}}_{u,n}$ and $\varsigma^{\textrm{NLoS}}_{u,n}$ are, respectively, the additional path losses at the LoS and NLoS air-to-ground links between DBS $n$ to user $u$. The value of $\varsigma^{\textrm{LoS}}_{u,n}$ and $\varsigma^{\textrm{NLoS}}_{u,n}$ follow Gaussian distributions with different parameters $\left(\mu_{\textrm{LoS}}, \delta^2_{\textrm{LoS}}\right)$ and $\left(\mu_{\textrm{LoS}}, \delta^2_{\textrm{LoS}}\right)$, respectively. Note that the path loss values between DBS $n$ and user $u$ are considered to be stable with DBS $n$'s movement,  as the distance between DBS $n$ and user $u$ will only experience small changes when DBS $n$ flies within the service area. 
The signal-to-noise ratio (SNR) at the LoS and NLoS links between DBS $n$ and user $u$, will thus be 

 \begin{equation}\label{eq:dr}
\setlength{\abovedisplayskip}{-5 pt}
\begin{split}
&\gamma^{\textrm{LoS}}_{u, n}=\frac{P}{N_0B10^{\frac{h^{\textrm{LoS}}_{u, n}}{20}}},\\%G^r_t\left(\theta^{n,u}_{i}\right),
&\gamma^{\textrm{NLoS}}_{u, n}=\frac{P}{N_0B10^{\frac{h^{\textrm{NLoS}}_{u, n}}{20}}},
\end{split}
\end{equation} 
where $P$ represents the transmit power of user $u$, which is assumed to be equal for all users. $N_0$ is the noise power spectral density, and $B$ is the RB bandwidth (equal for all RBs). Note that, the DBSs use different frequency bands to serve the users so as to avoid interference among the communication links.
The data rate at the link between DBS $n$ and user $u$ will then be $c_{u, n}= \beta^{\textrm{LoS}}_{u,n}{B}\log \left( {1 + \gamma^{\textrm{LoS}}_{u, n}} \right)+\beta^{\textrm{NLoS}}_{u,n}{B}\log \left( {1 + \gamma^{\textrm{NLoS}}_{u, n}} \right)$, where $\beta^{\textrm{LoS}}_{u,n}=\left[1+\varphi{\exp}\left(-\phi \frac{180}{\pi}\theta_{n,u}+\varphi\phi\right)\right]^{-1}$ is the probability of having a LoS link between DBS $n$ and user $u$, $\beta^{\textrm{NLoS}}_{u,n}=1-\beta^{\textrm{LoS}}$ is the probability of having a NLoS link between DBS $n$ and user $u$. Here, $\varphi$ and $\phi$ are constant values that depend on the studied communication environments, while $\theta_{n,u}$ is the elevation angle between DBS $n$ and user $u$.

\vspace{-0.2cm}
\subsection{Utility Function Model}
 
  In the studied scenario, the goal of the dispatched DBSs is to cover all access requests from ground users. In such a case, the utility of each DBS is defined as the \emph{successful service rate}, which captures the fraction of users being served by a given DBS in a given time period. Note that, when DBS $n$ arrives at a cluster, it will only serve the user requests that were not served by the DBSs that arrived at this cluster before DBS $n$. For the special case in which multiple DBSs arrive at a cluster at the same time, only one DBS will serve the entire cluster, and the remaining DBSs will directly proceed toward other clusters. Thus, the successful service rate achieved at DBS $n$ by serving cluster $\xi_{n,k}$ is given by  
% \begin{equation}\label{eq:2}
% %\setlength{\abovedisplayskip}{-3 pt}
%%\setlength{\belowdisplayskip}{3 pt}
%\begin{split}
%\mu_{k}\left(\boldsymbol{\xi}\right) = \frac{\sum_{n\in\mathcal{N}}\sum_{u\in\mathcal{U}}\mathds{1}_{\left\{u\in \mathcal{U}_{{\xi}_{n,k}}, T-\tau_{n}^* \le t_u\le T-\tau_{n,k}\right\}}}{\sum_{u\in\mathcal{U}}\mathds{1}_{\left\{0\le t_u\le T\right\}}}
%%\mu\left(\boldsymbol{\xi}\right) = \sum_{l=1}^{L}\sum_{n\in\mathcal{N}}\frac{\sum_{u\in\mathcal{U}}\mathds{1}_{\left\{u\in \mathcal{U}_{l}, t_u\le T-\tau^*\right\}}}{\sum_{u\in\mathcal{U}}\mathds{1}_{\left\{0\le t_u\le T\right\}}}.
%\end{split}
%\end{equation} 
 \begin{equation}\label{eq:ssrate}
\begin{split}
\mu_{n,k}\left(\boldsymbol{\xi}\right)=\frac{\sum_{u\in\mathcal{U}}\mathds{1}_{\left\{u\in \mathcal{U}_{{n, k}}, T-\tau_{n}^* \le t_u\le T-\tau_{n,k}\right\}}}{\sum_{u\in\mathcal{U}}\mathds{1}_{\left\{0\le t_u\le T\right\}}},
\end{split}
\end{equation} 
where $\boldsymbol{\xi}=[\boldsymbol{\xi}_1, \boldsymbol{\xi}_2,\ldots,\boldsymbol{\xi}_N]$ is the matrix of trajectories of the DBSs, $\mathcal{U}_{{n,k}}$ is the set of active users in cluster ${\xi}_{n,k}$. $\tau_{n,k}$ is the time duration that DBS $n$ is allowed to keep flying with its remaining energy level, after successfully serving cluster ${{\xi}_{n,k}}$. Note that, hereinafter, the time duration $\tau_{n,k}$ is called the \emph{available service time} of DBS $n$ at step $k$. In (\ref{eq:ssrate}), $\mathds{1}_{\left\{x\right\}}=1$ when $x$ is true, otherwise, $\mathds{1}_{\left\{x\right\}}=0$. Here, $\sum_{u\in\mathcal{U}}\mathds{1}_{\left\{0\le t_u\le T\right\}}$ is the number of active users within the studied time duration, and $\sum_{u\in\mathcal{U}}\mathds{1}_{\left\{u\in \mathcal{U}_{{n,k}}, T-\tau_{n}^* \le t_u\le T-\tau_{n,k}\right\}}$ is the number of active users served by DBS $n$ in cluster $\xi_{n,k}$. $\tau_{n}^*= \mathop {\min }\limits_{n'\in \mathcal{N}_n} \tau_{n', k'}$ is the available service time of the last DBS that arrived at cluster ${\xi}_{n, k}$ before DBS $n$. $T-\tau_{n}^* \le t_u\le T-\tau_{n,k}$ represents that user $u$ requests data before DBS $n$'s arrival, and have not served by any other DBSs, as shown in Fig. \ref{Fig. new}.
Here, $\mathcal{N}_{n}=\left\{n'\left| {n'\in\mathcal{N}\setminus n, {\xi}_{n',k'}} = {{\xi}_{n,k}}, \tau_{n',k'}\ge \tau_{n,k} \right.\right\}$ is the set of DBSs that arrive at cluster ${\xi}_{n, k}$ before DBS $n$, where ${\xi}_{n', k'}={\xi}_{n,k}$ implies that cluster ${\xi}_{n,k}$ is also the $k'$-th cluster served by DBS $n'$. DBS $n$'s available service time upon its arrival at cluster ${{\xi}_{n,k}}$ on trajectory $\boldsymbol{\xi}_{n}$ is $\tau_{n,k}=T-\sum_{\kappa=0}^{k-1}\frac{d_{n,{\kappa},{\kappa+1} }}{V}-\sum_{\kappa=1}^{k-1}D^*_{{n, \kappa}}$, with $d_{{n, \kappa},{\kappa+1}}$ being the distance between cluster ${\xi}_{n, \kappa}$ and ${\xi}_{n, \kappa+1}$. 
%$\tau^*=\!\!\!\!\!\!\mathop {\min }\limits_{n\in \left\{n\left| {n\in\mathcal{N}, l\in\boldsymbol{\xi}_{n}} \right.\right\}} \tau_{n,l_{\boldsymbol{\xi}_{n}, k}}$ is available service time of the last arrived DBS, with $\tau_{n,l_{\boldsymbol{\xi}_{n}, k}}$ being the available service time of DBS $n$ upon its arrival at cluster $l_{\boldsymbol{\xi}_{n}, k}$ on trajectory $\boldsymbol{\xi}_{n}\in\mathcal{E}$. 
Here, $\frac{d_{n,{\kappa},{\kappa+1} }}{V}$ is the time needed by DBS $n$ to travel in an SLF from cluster ${\xi}_{n, \kappa}$ to cluster ${\xi}_{n, \kappa+1}$. $D^*_{{n, \kappa}}$ is the time needed by DBS $n$ for hovering while serving the users in cluster ${\xi}_{n, \kappa}$, and this hovering time is given by 
 \begin{equation}\label{eq:3}
\begin{split}
 D^*_{{n, \kappa}}=\mathop {\max }\limits_{u\in \mathcal{U}^*_{{n, \kappa}}} D_{u, n}-\frac{2d_r}{V},
\end{split}
\end{equation}
where $D_{u, n}=\frac{b_u}{c_{u, n}}$ is the transmission delay of user $u$,  when being served by DBS $n$. $\mathop {\max }\limits_{u\in \mathcal{U}_{{n, \kappa}}} D_{u, n}$ is the time that DBS $n$ used to serve all the users in cluster ${\xi}_{n, \kappa}$. $ D^*_{{n, \kappa}}$ is the time DBS $n$ consumes for SCF hovering, while $\frac{2d_r}{V}$ is the time DBS $n$ consumes its SLF travel within the service area. We define $\mathcal{U}^*_{{n, \kappa}}=\left\{u\left| {u\in\mathcal{U}_{{n, \kappa}}, T-\tau_{n}^* \le t_u\le T-\tau_{n,\kappa}} \right.\right\}$ as the set of active users that can be served by DBS $n$ in cluster ${{\xi}_{n, \kappa}}$. %Thus, the DBSs' available service time upon its arrival at cluster $l_{\boldsymbol{\xi}_{n}, k}$ on trajectory $\boldsymbol{\xi}_{n}$ is $\tau_{n,l_{\boldsymbol{\xi}_{n}, k}}=T-\sum_{\kappa=0}^{k-1}\frac{d_{l_{\boldsymbol{\xi}_{n}, \kappa},l_{\boldsymbol{\xi}_{n}, \kappa+1}}}{V}-\sum_{\kappa=1}^{k-1}D^*_{l_{\boldsymbol{\xi}_{n}, \kappa}}$, with $d_{l_{\boldsymbol{\xi}_{n}, k},l_{\boldsymbol{\xi}_{n}, k+1}}$ being the distance between cluster $l_{\boldsymbol{\xi}_{n}, k}$ and $l_{\boldsymbol{\xi}_{n}, k+1}$. In particular, with trajectories $\boldsymbol{\xi}$, the fraction of users being served by the DBS $n$ at step $k$, defined as \emph{successful service rate at DBS $n$}, is given by $\mu_{n,k}\left(\boldsymbol{\xi}\right) = \frac{\sum_{u\in\mathcal{U}}\mathds{1}_{\left\{u\in \mathcal{U}_{l_{\boldsymbol{\xi}_{n}, k}}, T-\tau_{n}^* \le t_u\le T-\tau_{n,k}\right\}}}{\sum_{u\in\mathcal{U}}\mathds{1}_{\left\{0\le t_u\le T\right\}}}$, where $\tau_{n}^*= \!\!\!\!\!\!\mathop {\min }\limits_{n'\in \left\{n'\left| {n'\in\mathcal{N}/n, l_{\boldsymbol{\xi}_{n}, \kappa} \in \boldsymbol{\xi}_{n'}}, \tau_{n',l_{\boldsymbol{\xi}_{n}, k}}\ge \tau_{n, l_{\boldsymbol{\xi}_{n}, k}} \right.\right\}} \tau_{n',l_{\boldsymbol{\xi}_{n}, k}}$ is available service time of the DBS that arrives at cluster $l_{\boldsymbol{\xi}_{n}, \kappa}$ just before DBS $n$.

\begin{figure}
\setlength{\belowcaptionskip}{-3pt}
\setlength{\abovecaptionskip}{-5pt} 
  \centering
  \includegraphics[width=9 cm]{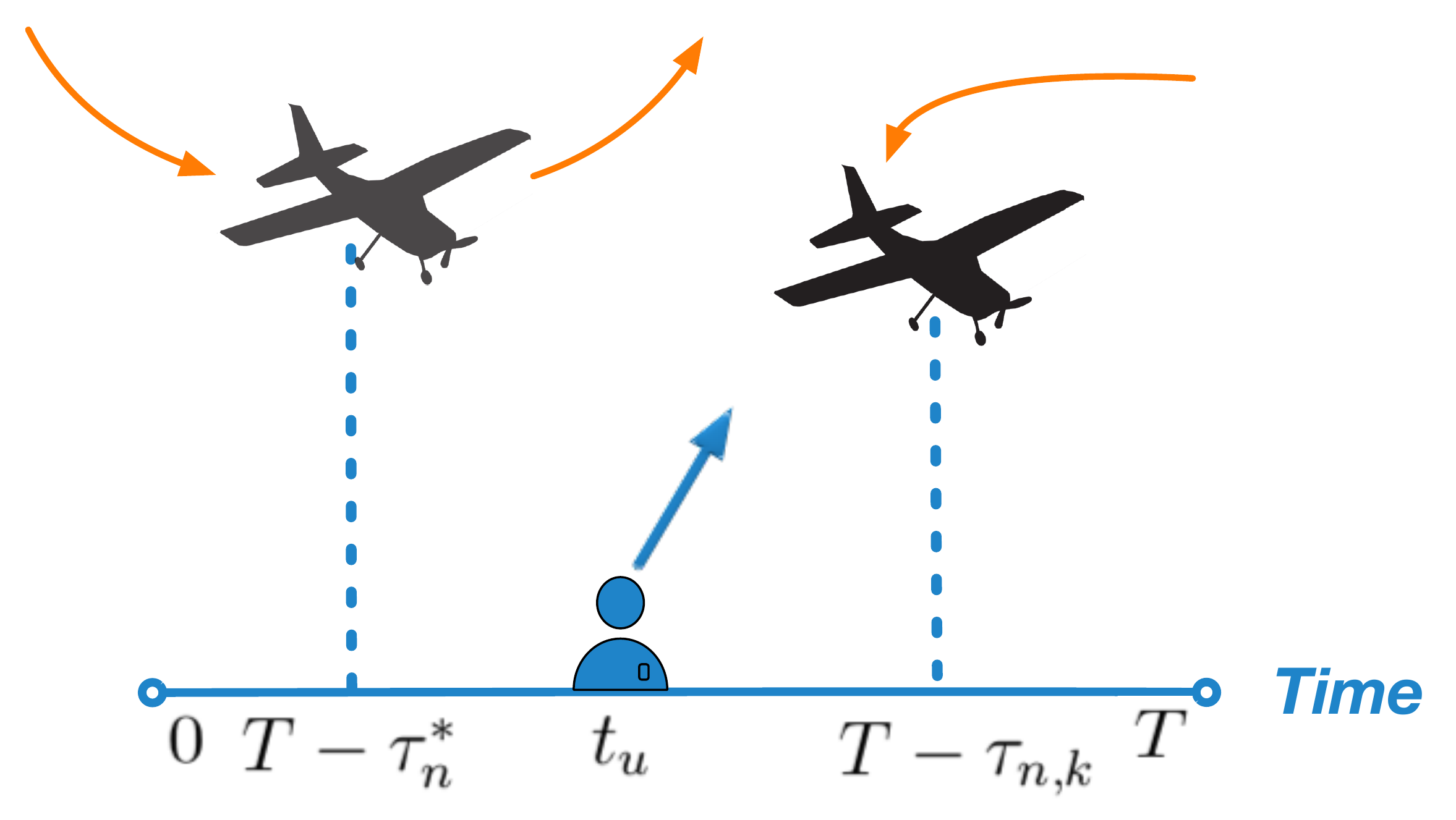}
  \caption{\footnotesize{Snapshot of one user access request.}}
  \label{Fig. new}
  \centering
  \vspace{-0.8cm}
\end{figure}

\vspace{-0.16cm}
\subsection{Problem Formulation}
In our system model, a group of energy-constrained and myopic DBSs fly independently to cover all the access requests from ground users. 
The successful service rate achieved by the DBSs is defined as a \emph{team utility}, and is given by
 \begin{equation}\label{eq:utility}
\begin{split}
G\left(\boldsymbol{\xi}\right)=\sum^	K_{k=1}\sum_{n\in\mathcal{N}}{\mu}_{n,k}\left(\boldsymbol{\xi}\right).
\end{split}
\end{equation}
The goal of the DBSs is to find optimal trajectories that maximize the expected team utility. Next, we first define expected team utility and then we introduce the optimization problem. Let $\pi_n\left(\xi\left| {{\xi}_{n,k}},\tau_{n,k} \right.\right)$ be the strategy of DBS $n$, defined as the probability that DBS $n$ moves toward cluster $\xi\in\mathcal{C} \cup \left\{O\right\}$, after successfully serving cluster ${{\xi}_{n,k}}$ with {available service time} $\tau_{n,k}$, and let $\boldsymbol{\pi}=\left[\pi_n\left(\xi\left| {{\xi}_{n,k}},\tau_{n,k} \right.\right)\right]_{n\in\mathcal{N}, k \in\mathcal{K}}$ be the vector of strategies of all DBSs. Then the expected team utility is defined as 
 \begin{equation}\label{eq:eutility}
 \setlength{\abovedisplayskip}{3 pt}
 \setlength{\belowdisplayskip}{3 pt}
\begin{split}
\overline G\left(\boldsymbol{\pi}\right)= \sum_{\boldsymbol{\xi}\in\mathcal{E}}G\left(\boldsymbol{\xi}\right)\prod\limits^{N}_{n=1}\prod\limits^{K}_{k=1}\pi_n\left(\xi\left| {{\xi}_{n,k}},\tau_{n,k} \right.\right).
\end{split}
\end{equation}
As such, the trajectory design problem can be formulated as
\addtocounter{equation}{0}
\begin{equation}\label{opt}
%\begin{split}
 \max_{\boldsymbol{\pi}}  \overline G\left(\boldsymbol{\pi}\right),
%\end{split}
\end{equation}
\vspace{-0.8cm}
\begin{align}\label{c1}
\setlength{\abovedisplayskip}{-3 pt}
&\;\;\;\;\rm{s.\;t.}\scalebox{1}{$\;\;\;\;  \sum_{\boldsymbol{\xi}\in\mathcal{E}}\prod\limits^{N}_{n=1}\prod\limits^{K}_{k=1}\pi_n\left(\xi\left| {{\xi}_{n,k}},\tau_{n,k} \right.\right)=1, $} \tag{\theequation a}\\
&\;\;\;\;\;\;\;\;\scalebox{1}{$\;\;\;\;\;\sum_{{a}_{n,k}\in\mathcal{C} \cup \left\{O\right\}}\pi_n\left(\xi\left| {{\xi}_{n,k}},\tau_{n,k} \right.\right)=1, \forall n\in \mathcal{N},  \boldsymbol{\xi}\in \mathcal {E}, k \in \mathcal{K},$} \tag{\theequation b}\\
&\;\;\;\;\;\;\;\;\scalebox{1}{$\;\;\;\;\;0\le \pi_n\left(\xi\left| {{\xi}_{n,k}},\tau_{n,k} \right.\right)\le 1, \forall n\in \mathcal{N}, \boldsymbol{\xi}\in \mathcal {E}, k \in \mathcal{K},$} \tag{\theequation c}
\end{align}
where $\mathcal{E}$ is the set of all possible trajectories of the DBSs. Here, (\ref{opt}a) means that the DBSs must choose trajectories from $\mathcal{E}$.  (\ref{opt}b) indicates that, within the considered time duration, each DBS $n$ must choose to serve one cluster in $\mathcal{C}$ or return to the origin. The optimal solution of problem (\ref{opt}) guarantees a maximum expected team utility at the DBSs, and it is called a \emph{team optimal strategy}. Here, we note that the use of traditional optimization algorithms, such as branch and bound or nonlinear programming, is not suitable to solve (\ref{opt}), as problem (\ref{opt}) is non-convex, and the successful service rate ${\mu}_{n,k}\left(\boldsymbol{\xi}\right)$ achieved at each DBS is unpredictable with the values of $\boldsymbol{b}$ and $\boldsymbol{t}$ following unknown distributions. 
 Moreover, traditional machine learning algorithms such as Q learning, policy gradient, and echo state networks (ESN) \cite{watkins1992q, chen2019joint, sutton2000policy} that have been previously used to solve complex optimization problem are also not suitable to solve the problem (\ref{opt}). This is because those algorithms must be manually adjusted to solve their training tasks, and they cannot find optimal strategies for the DBSs in unseen environments. 
To solve the non-convex problem (\ref{opt}) formulated in dynamic, unpredictable environments, we propose a distributed meta-trained VD-RL algorithm that finds team optimal strategies by meta training a distributed VD-RL solution equipped with a primary estimation on unseen environments. The distributed VD-RL algorithm updates the DBSs' strategies and estimation on the strategy outcomes at each individual DBS, based on only this DBS's actions and states. The meta training procedure finds an initialization of the policy and value functions for the VD-RL solution. This initialization is close to all optimal policy and value functions at all possible environments and, hence, it enables the VD-RL solution quickly converge in the dynamic environments and reduces the VD-RL solution's cost on time, energy, and the drone hardware.
The proposed meta-trained VD-RL algorithm is introduced in the next section.

\vspace{-0.3cm}
\section{Proposed Value Decomposition-Reinforcement Learning Algorithm with Meta Training}

We now introduce a distributed meta-trained VD-RL algorithm, that merges the concept of value decomposition network \cite{sunehag2017value}, model agnostic meta-learning \cite{finn2017model}, with the policy gradient (PG) framework. The traditional PG algorithm can find the optimal trajectory for a single DBS. However, PG cannot find a team optimal strategy for a group of DBSs, as it will direct all the DBSs to one trajectory. In order to design trajectories for multiple DBSs, an algorithm that can reach a team optimal strategy for the DBSs must be proposed. Moreover, to prepare the DBSs for unseen environments, the proposed algorithm must not be overfitted to its training tasks, and it should converge to a team optimal strategy when the team utility function in (\ref{eq:utility}) changes.
To address these challenges, we propose a meta-trained VD-RL that coordinates a group of DBSs in various environments. Next, we first introduce the proposed VD-RL algorithm which solves the non-convex problem (\ref{opt}). Then, we explain how to use meta-learning approaches to meta train this VD-RL algorithm in order to allow it to cope with various environments.

\vspace{-0.2cm}
\subsection{Value Decomposition based Reinforcement Learning Algorithm}
Next, we first introduce the components of the VD-RL algorithm. Then, we explain how the VD-RL algorithm can decompose problem (\ref{opt}) to problems that can be solved at each individual DBS. Finally, we introduce the procedure of using the proposed VD-RL algorithm to solve the non-convex problem in (\ref{opt}). 
\subsubsection{Value decomposition based reinforcement learning components}
The proposed VD-RL algorithm consists of seven components
\begin{itemize}
\item \emph{Agents}: The agents in VD-RL are the the DBSs in set $\mathcal{N}$. 

\item \emph{States}: Each agent has a state that consists of both its location, represented by the cluster it currently serves i.e. ${\xi}_{n, k}$, and its energy level, which is captured by the time $\tau_n,k$ that this DBS still has in order to return to the origin. The state of each DBS $n$ at step $k$ is given by $\boldsymbol{s}_{n,k}=\left[{\xi}_{n, k}, {\tau}_{n,k}\right]$. The set of states at all DBSs is {$\mathcal{S}=\left\{\boldsymbol{S}_0, \boldsymbol{S}_1,\ldots,\boldsymbol{S}_K\right\}$}, with {$\boldsymbol{S}_k=\left[\boldsymbol{s}_{1, k},\boldsymbol{s}_{2,k},\ldots,\boldsymbol{s}_{N,k}\right]$} being the matrix of the DBSs' states at step $k$.

%$\mathcal{S}=\left\{\boldsymbol{S}_0, \boldsymbol{S}_1,\ldots,\boldsymbol{S}_K\right\}$ is the set of states at the DBSs, in which $\boldsymbol{S}_k=\left[\boldsymbol{s}_{1, k},\boldsymbol{s}_{2,k},\ldots,\boldsymbol{s}_{N,k}\right]$ defines the DBSs' states at step $k$ with $\boldsymbol{s}_{n,k}=\left[{a}_{n, k-1}, {\tau}_{n,k}\right]$.

\item \emph{Actions}: The action of each agent is the cluster that it seeks to serve, or the origin location that it will return to after serving several clusters. The action chosen by DBS $n$ at step $k$ is given by $a_{n,k}\in\mathcal{C}\cup\left\{O\right\}$, while the vector of all DBSs' actions at step $k$ is $\boldsymbol{a}_k=\left[a_{1,k},a_{2,k},\ldots,a_{N,k}\right]$.

% the set of available actions for DBS $n$ is $\mathcal{C}\cup\left\{O\right\}$, which is the clusters that DBS $n$ can choose to serve and the origin location DBS $n$ returns to after serving the cluster.  %$\boldsymbol{\xi}$ represents the sequence of actions that have been taken by each of the DBSs before reaching the terminal step .

%\item $\boldsymbol{p}$ defines the state transition probabilities.  In particular, $p\left(\boldsymbol{S}' \left| {\boldsymbol\boldsymbol{S}_k, \boldsymbol{a}_k} \right.\right)$ defines the probability  transitioning to state $\boldsymbol{S}'$ given it was in state $\boldsymbol{S}_k$ and the DBSs took actions $\boldsymbol{a}_k=\left[a_{1,k}, a_{2,k},\ldots,a_{N,k}\right]$, with $a_{n,k}$ being DBS $n$'s action at step $k$. 

%\item $\boldsymbol{\xi}$ represents the sequence of actions that have been taken by each of the DBSs before reaching the terminal step .

%\item $\boldsymbol{\mu}\left(\boldsymbol{\xi}\right)= \left[\sum_{n\in\mathcal{N}}{\mu}_{n,1}\left(\boldsymbol{\xi}\right), \sum_{n\in\mathcal{N}}{\mu}_{n,2}\left(\boldsymbol{\xi}\right), \ldots, \sum_{n\in\mathcal{N}}{\mu}_{n,K}\left(\boldsymbol{\xi}\right)\right]$ is the vector of utilities of DBSs at each step , in which $\sum_{n\in\mathcal{N}}{\mu}_{n,k}\left(\boldsymbol{\xi}\right)$ is denoted the \emph{stage payoff} of the DBSs at step $k$ .  %Here, $\boldsymbol{\mu}_{k}\left(\boldsymbol{\xi}\right)=\left[{\mu}_{1,k}\left(\boldsymbol{\xi}\right), {\mu}_{2,k}\left(\boldsymbol{\xi}\right),\ldots, {\mu}_{N,k}\left(\boldsymbol{\xi}\right)\right]$ is vector

\item \emph{Strategy}: The strategy of each DBS is defined as the probability of choosing a given action ${a}_{n,k}\in\mathcal{C}\cup\left\{O\right\}$ at a given state $\boldsymbol{s}_{n,k}$ and is denoted by $\pi_n\left({a}_{n,k}\left| {\boldsymbol{s}_{n,k}} \right.\right)$. $\boldsymbol{\pi}=\left[\pi_n\left({a}_{n,k}\left| {\boldsymbol{s}_{n,k}} \right.\right)\right]_{n\in\mathcal{N}, k \in\mathcal{K}}$ is the vector of strategies of all DBSs. 

\item \emph{Policy function}: We define a policy function $\boldsymbol{\pi}_{\boldsymbol{\theta}_{a,n}}$ that is a deep neural network parametrized by $\boldsymbol{\theta}_{a,n}$ and is used to generate the strategy of DBS $n$. This policy function $\boldsymbol{\pi}_{\boldsymbol{\theta}_{a,n}}$ takes DBS $n$'s state as an input and outputs a strategy for DBS $n$ at this state. 

\item \emph{Reward}: The reward of each DBS measures the benefit of a selected action. In particular, aiming at maximizing coverage in the considered area, the reward of each DBS is defined as the successful service rate achieved by all the DBSs, which is given by\footnote{Note that, DBS $n$'s successful service rate $\mu_{n,k}$ defined in (\ref{eq:ssrate}), depends on DBS $n$'s current action $a_{n,k}$, state $\boldsymbol{s}_{n,k}$, and the actions taken by other DBSs before DBS $n$'s arrival at cluster $\xi_{n,k}$. As the DBSs operate asynchronously in the wireless environment as shown in Fig. \ref{Fig. 2}, one DBS may reach step $k'  \ne k$ on its trajectory at the time DBS $n$ arrives at its step $k$. In other words, at time epoch $T-\tau_{n,k}$, the DBSs will reach different steps of their trajectories. The set of DBSs' states at time epoch $T-\tau_{n,k}$ is denoted as $\mathcal{B}_{T-\tau_{n,k}}$. In this case,
 the team stage reward resulting from action $\boldsymbol{a}_k$ should be $r\left(\boldsymbol{a}_k\left|\cup_{n\in\mathcal{N}} \mathcal{B}_{T-\tau_{n,k}}\right.\right)$. Hereinafter, we denote it as $r\left(\boldsymbol{a}_k\left| \boldsymbol{S}\right.\right)$ for simplicity. It is worthy noting that, within the proposed VD-RL algorithm,  the DBSs only record the number of served active users at each step on their trajectories.} $r\left(\boldsymbol{a}_k\left|\mathcal{S}\right.\right) = \sum_{n\in\mathcal{N}}{\mu}_{n,k}\left(\boldsymbol{\xi}\right)$. Here, different from traditional RL algorithms in which each agent only maximizes its own reward, our proposed VD-RL algorithm enables each DBS to maximize the reward of all DBSs, which is also called the team stage reward.  

\item \emph{Value function}: We define a value function $V\left(\boldsymbol{S}_k\right)$ that is a deep neural network used to estimate the DBSs' achievable future rewards at every state $\boldsymbol{S}_k$. In particular, within the proposed VD-RL algorithm, the goal of the DBSs is to find a team optimal strategy that maximizes the expected team utility in (\ref{eq:utility}). To this end, the DBSs must consider the current and also future rewards that can be achieved at every state. Thus, at every state $\boldsymbol{S}_k$, the DBSs seeks to maximize a discounted future reward $\sum\limits_{n\in\mathcal{N}}\sum\limits^{K}_{\kappa=k}\gamma^{k-\kappa+1}r\left(\boldsymbol{a}_{\kappa}\left|\mathcal{S}\right.\right)$, which is estimated by value function $V\left(\boldsymbol{S}_k\right)$, with $\gamma$ being the discounted factor. %Note that, as value function $V\left(\boldsymbol{S}_k\right)$ estimates the future rewards of all DBSs, it is also denoted as \emph{value}.
%\item ${\gamma}$ is the discount factor.
\end{itemize}

\begin{figure}
  \centering
  \includegraphics[width=12 cm]{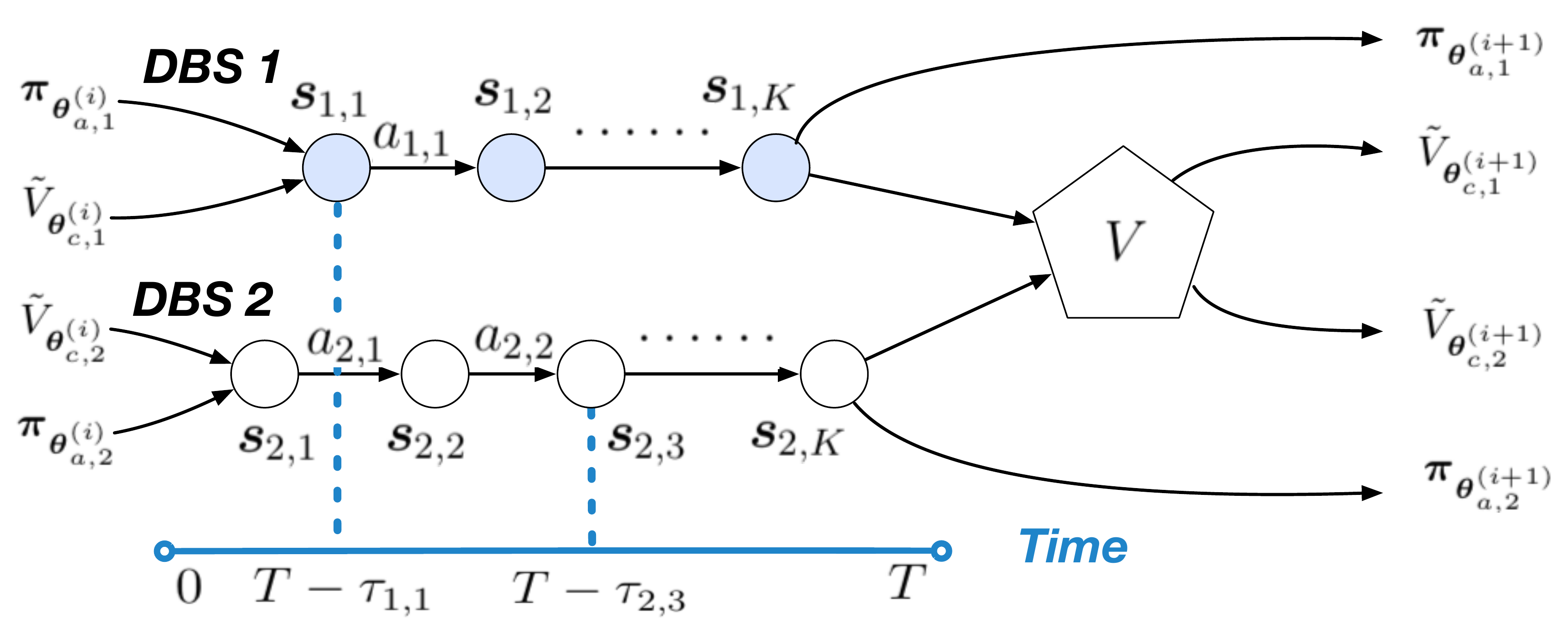}
  \caption{\footnotesize{Illustration of the proposed VD-RL algorithm.}}
  \label{Fig. 2}
  \centering
  \vspace{-0.7cm}
\end{figure}
%The VD-RL algorithm provide an approach to find a team optimal strategy in one DTG. 
In VD-RL, the DBSs interact with the unknown wireless environment by selecting actions in $\mathcal{C}\cup\left\{O\right\}$ based on the strategy generated by policy function $\boldsymbol{\pi}_{\boldsymbol{\theta}_{a,n}}$. In particular, the DBSs, select actions step by step and receive the team stage rewards as feedback. An \emph{experience} defined in vector $\boldsymbol{e}_n=\left[\boldsymbol{s}_{n,1}, {a}_{n,1}, r\left(\boldsymbol{a}_1\left|\mathcal{S}\right.\right), \ldots, \boldsymbol{s}_{n,K}, {a}_{n,K}, r\left(\boldsymbol{a}_K\left|\mathcal{S}\right.\right)\right]$ is collected by each DBS $n\in\mathcal{N}$. Each DBS $n$ then updates its value function with its collected experience $\boldsymbol{e}_n$ and updates its strategy to maximize future rewards estimated by $V\left(\boldsymbol{S}_k\right)$. In particular, the DBSs will update its estimation on future rewards after they return to the origin. At the origin, each DBS $n$ only needs to know the number of active users served by all the DBSs from their last experience so as to update its estimation on future rewards, it does not have to know the actions taken by other DBSs. However, value function $V\left(\boldsymbol{S}_k\right)$ still depends on the states of all DBSs. Hence, each DBS may not be able to train its value function individually. Therefore, we introduce a value decomposition method using which each DBS can train its value function individually. 

\subsubsection{Value decomposition} The VD-RL algorithm needs to find a value function accurately estimate future rewards, and, thus, seeks to minimize the temporal difference (TD) error metric that is defined as follows\cite{baird1993advantage}:
 \begin{equation}\label{eq:tderror}
 \setlength{\abovedisplayskip}{-2 pt}
\begin{split}
A\left(\boldsymbol{a}_{k}, {\boldsymbol{S}_{k}}\right) &= N r\left(\boldsymbol{a}_k\left|\boldsymbol{S}\right.\right)+ \gamma V\left(\boldsymbol{S}_{k+1}\right)- V\left(\boldsymbol{S}_k\right)
\end{split}
\end{equation} 
% \begin{equation}\label{eq:tderror}
% \setlength{\abovedisplayskip}{-2 pt}
%%\setlength{\belowdisplayskip}{3 pt}
%\begin{split}
%A\left(\boldsymbol{a}_{k}, {\boldsymbol{S}_{k}}\right) &= N r\left(\boldsymbol{a}_k\left|\boldsymbol{S}\right.\right)+  \sum\limits^{K}_{\kappa=1}\gamma^{\kappa}r\left(\boldsymbol{a}_{k+\kappa}\left|\mathcal{S}\right.\right)- V\left(\boldsymbol{S}_k\right)
%\end{split}
%\end{equation} 
where TD error $A\left(\boldsymbol{a}_{k}, {\boldsymbol{S}_{k}}\right)$ is called the \emph{team advantage} of the DBSs at state ${\boldsymbol{S}_{k}}$ with action $\boldsymbol{a}_{k}$. The team advantage measures the difference between the DBSs' future reward estimated by $Nr\left(\boldsymbol{a}_k\left|\boldsymbol{S}\right.\right)+ \gamma V\left(\boldsymbol{S}_{k+1}\right)$, and the one estimated by $V\left(\boldsymbol{S}_k\right)$. In essence, the team advantage keeps improving the value function $V\left(\boldsymbol{S}_k\right)$'s estimation accuracy using new sampled experiences.
Moreover, the team advantage also reveals how actions in $\boldsymbol{a}_{k}$ are better than other action selections by showing the difference between the future reward reached by actions in $\boldsymbol{a}_{k}$ and the estimated future reward $V\left(\boldsymbol{S}_k\right)$. 
However, to enable the DBSs to update their value functions individually, one must divide the value of function $V\left(\boldsymbol{S}_k\right)$ among all DBSs. The value allocated to each DBS must reinforce team beneficial actions, and weaken all other actions at each DBS.
Our VD-RL algorithm implements this value decomposition across the DBSs. In particular, the proposed VD-RL algorithm decomposes value $V\left(\boldsymbol{S}_k\right)$ under the assumption that this value is the sum of the values attributed to each DBS
 \begin{equation}\label{eq:vd}
  \setlength{\abovedisplayskip}{2 pt}
\setlength{\belowdisplayskip}{2 pt}
\begin{split}
V\left(\boldsymbol{S}_k\right) = \sum_{n\in\mathcal{N}}\tilde{V}_{\boldsymbol{\theta}_{c,n}}\left(\boldsymbol{s}_{n,k}\right),
\end{split}
\end{equation} 
where $\tilde{V}_{\boldsymbol{\theta}_{c,n}}\left(\boldsymbol{s}_{n,k}\right)$ is the individual value function parametrized by $\boldsymbol{\theta}_{c,n}$ at DBS $n$. Thus, the team advantage $A\left(\boldsymbol{a}_{k}, {\boldsymbol{S}_{k}}\right)$ be expressed as
 \begin{equation}\label{eq:vd2}
\begin{split}
A\left(\boldsymbol{a}_{k}, {\boldsymbol{S}_{k}}\right) &= N r\left(\boldsymbol{a}_k\left|\boldsymbol{S}\right.\right)+ \gamma \sum_{n\in\mathcal{N}}\tilde{V}_{\boldsymbol{\theta}_{c,n}}\left(\boldsymbol{s}_{n,k+1}\right)- \sum_{n\in\mathcal{N}}\tilde{V}_{\boldsymbol{\theta}_{c,n}}\left(\boldsymbol{s}_{n,k}\right).
\end{split}
\end{equation} 
Based on the assumption in (\ref{eq:vd}), the DBSs can update their policy and value functions independently.

%The DBSs will collect experiences defined in a vector $\boldsymbol{e}=\left[\boldsymbol{S}_{1}, \boldsymbol{a}_{1}, r\left(\boldsymbol{a}_1\left|\boldsymbol{S}_1\right.\right), \ldots, \boldsymbol{S}_{K}, \boldsymbol{a}_{K}, r\left(\boldsymbol{a}_K\left|\boldsymbol{S}_K\right.\right)\right]$ at a realization of the user service request, based on which the DBSs update their value function and then their policy functions. 

\subsubsection{Value decomposition based reinforcement learning solution}
When searching for the team optimal strategy that solves problem (\ref{opt}), % the VD-RL algorithm needs to find policy and value functions based on Proposition \ref{proposition2}. %
the VD-RL algorithm needs to find optimal value functions that accurately estimate future rewards at every state, as well as the optimal policy functions that can always yield the actions leading the highest future rewards at each DBS. 
In particular, the update of the individual value function $\tilde{V}_{\boldsymbol{\theta}_{c,n}}\left(\boldsymbol{s}_{n,k}\right)$ of each DBS $n$ is given by  %$\nabla_{\boldsymbol{\theta}_{c,n}}A^2\left(\boldsymbol{a}_{k}, {\boldsymbol{S}_{k}}\right)$
 \begin{equation}\label{eq:vupdate}
\begin{split}
%\tilde{V}_{\boldsymbol{\theta}_{c,n}}\left(\boldsymbol{s}_{n,k}\right) = \tilde{V}_{\boldsymbol{\theta}_{c,n}}\left(\boldsymbol{s}_{n,k}\right)+\alpha_c \nabla_{\boldsymbol{\theta}_{c,n}} \sum^K_{k=1} {A\left(\boldsymbol{a}_{k}, {\boldsymbol{S}_{k}}\right)},
\boldsymbol{\theta}^{\left(i+1\right)}_{c,n}&=\boldsymbol{\theta}^{\left(i\right)}_{c,n}-\alpha^{\left(i\right)}_c\nabla_{\boldsymbol{\theta}_{c,n}} \sum^K_{k=1}A^2\left(\boldsymbol{a}_{k}, {\boldsymbol{S}_{k}}\right) \\
&= \boldsymbol{\theta}_{c,n}^{\left(i\right)}-\alpha^{\left(i\right)}_c\nabla_{\boldsymbol{\theta}_{c,n}} \sum^K_{k=1} \left(N r\left(\boldsymbol{a}_k\left|\boldsymbol{S}\right.\right)+ \gamma\sum_{n\in\mathcal{N}} \tilde{V}_{\boldsymbol{\theta}_{c,n}}\left(\boldsymbol{s}_{n,k+1}\right)- \sum_{n\in\mathcal{N}}\tilde{V}_{\boldsymbol{\theta}_{c,n}}\left(\boldsymbol{s}_{n,k}\right)\right)^2\\
&=\boldsymbol{\theta}^{\left(i\right)}_{c,n}-2\alpha^{\left(i\right)}_c  \sum^K_{k=1} A\left(\boldsymbol{a}_{k}, {\boldsymbol{S}_{k}}\right)\nabla_{\boldsymbol{\theta}_{c,n}} \left(\gamma  \tilde{V}_{\boldsymbol{\theta}_{c,n}}\left(\boldsymbol{s}_{n,k+1}\right)-\tilde{V}_{\boldsymbol{\theta}_{c,n}}\left(\boldsymbol{s}_{n,k}\right) \right),
\end{split}
\end{equation} 
where $\boldsymbol{\theta}^{\left(i\right)}_{c,n}$ is the value function parameter of DBS $n$ and $\alpha^{\left(i\right)}_c$ is the value function parameter update step size, at the $i$-th VD-RL training iteration.

After precisely estimating the future team stage rewards using $V\left(\boldsymbol{S}_k\right)$, and attributing the value to each one of the DBS, the policy function of each DBS can be updated.  Let $\tilde {A}_n\left({a}_{n,k}, {\boldsymbol{s}_{n,k}}\right)=  r\left(\boldsymbol{a}_k\left|\boldsymbol{S}\right.\right)+ \gamma \tilde{V}_{\boldsymbol{\theta}_{c,n}}\left(\boldsymbol{s}_{n,k+1}\right)- \tilde{V}_{\boldsymbol{\theta}_{c,n}}\left(\boldsymbol{s}_{n,k}\right)$ be the individual advantage at BS $n$ at state $\boldsymbol{s}_{n,k}$ with action ${a}_{n,k}\in\boldsymbol{a}_k$. Hence, we have $A\left(\boldsymbol{a}_{k}, {\boldsymbol{S}_{k}}\right)=\sum_{n\in\mathcal{N}}\tilde {A}_n\left({a}_{n,k}, {\boldsymbol{s}_{n,k}}\right)$. 
%In particular, the VD-RL seeks to find the policy functions with which the DBSs always choose the actions resulting with highest individual advantage $\tilde {A}_n\left({a}_{n,k}, {\boldsymbol{s}_{n,k}}\right)$, decomposed from $A\left(\boldsymbol{a}_{k}, {\boldsymbol{S}_{k}}\right)$ as in
%
% \begin{equation}\label{eq:iadg}
%  \setlength{\abovedisplayskip}{-1 pt}
%%\setlength{\belowdisplayskip}{3 pt}
%\begin{split}
%A\left(\boldsymbol{a}_{k}, {\boldsymbol{S}_{k}}\right) &= \sum_{n\in\mathcal{N}} r\left(\boldsymbol{a}_k\left|\boldsymbol{S}\right.\right)+ \gamma \tilde{V}_{\boldsymbol{\theta}_{c,n}}\left(\boldsymbol{s}_{n,k+1}\right)- \tilde{V}_{\boldsymbol{\theta}_{c,n}}\left(\boldsymbol{s}_{n,k}\right),\\
%&=\sum_{n\in\mathcal{N}}\tilde {A}_n\left({a}_{n,k}, {\boldsymbol{s}_{n,k}}\right).
%\end{split}
%\end{equation} 
Based on the the policy gradient theorem \cite{sutton2000policy}, the update on DBS $n$'s policy function parameters is given by
% is updated in the direction of $\nabla_{\boldsymbol{\theta}_{a,n}} E_{\boldsymbol{\pi}_{\boldsymbol{\theta}_{a,n}}}\left(\tilde {A}_n\left({a}_{n,k}, {\boldsymbol{s}_{n,k}}\right)\right)$, which, based on the the policy gradient theorem \cite{sutton2000policy}, is approximated in the form of 
 \begin{equation}\label{eq:pupdate}
\begin{split}
%\pi_{\boldsymbol{\theta}_c}\left(a_{n,k}\left| {\boldsymbol{s}_{n,k}} \right.\right) = \pi_{\boldsymbol{\theta}_c}\left(a_{n,k}\left| {\boldsymbol{s}_{n,k}} \right.\right)+\alpha_a\sum^K_{k=1}\tilde {A}_n\left({a}_{n,k}, {\boldsymbol{s}_{n,k}}\right)\nabla_{\boldsymbol{\theta}_c}\log \pi_{\boldsymbol{\theta}_c}\left(a_{n,k}\left| {\boldsymbol{s}_{n,k}} \right.\right),
\boldsymbol{\theta}^{\left(i+1\right)}_{a,n}=\boldsymbol{\theta}_{a,n}^{\left(i\right)}+\alpha^{\left(i\right)}_a\sum^K_{k=1}\tilde {A}_n\left({a}_{n,k}, {\boldsymbol{s}_{n,k}}\right)\nabla_{\boldsymbol{\theta}_{a,n}}\log \pi_{\boldsymbol{\theta}_{a,n}}\left(a_{n,k}\left| {\boldsymbol{s}_{n,k}} \right.\right),
\end{split}
\end{equation}
where $\boldsymbol{\theta}_{a,n}^{\left(i\right)}$ is the policy function parameter at DBS $n$ and $\alpha^{\left(i\right)}_a$ is the  policy function parameter update step size, at the $i$-th VD-RL training iteration. 

  \begin{algorithm}[t]\footnotesize

\caption{VD-RL algorithm for trajectory design at one realization of user access request. }   

\label{alg:VD-RL}   
\setlength{\abovecaptionskip}{-40pt} 
\setlength{\belowcaptionskip}{-40pt}
\begin{algorithmic} [1] %è¿ä¸ª1 è¡¨ç¤ºæ¯ä¸è¡é½æ¾ç¤ºæ°å­  
\REQUIRE User locations, time constraints. \\ 
\vspace{2pt}  
\ENSURE Initialize value functions $\tilde{V}_{\boldsymbol{\theta}^{\left(1\right)}_{c,n}}$, and policy functions $\boldsymbol{\pi}_{\boldsymbol{\theta}^{\left(1\right)}_{a,n}}$, for $n \in\mathcal{N}$.\\%, such that $\Pr \left( {{a_n}} \right)=\Pr \left( {{b_m}} \right)=\frac{1}{2}$ satisfies equation (12).  \\  

\vspace{2pt}  
\FOR {VD-RL training epoch $i=1:I$} 
\vspace{2pt}  
\STATE The DBSs carry out an experience by generating a sequence of actions that are randomly selected based on current policy functions.
\vspace{2pt}  
\STATE Calculate advantage $A$ based on (\ref{eq:tderror}).
\vspace{2pt}  
\FOR {Each DBS $n=1:N$} 
\vspace{2pt}  
\STATE Update individual value function with (\ref{eq:vupdate}).%$\boldsymbol{\theta}^{\left(i+1\right)}_{c,n}=\boldsymbol{\theta}^{\left(i\right)}_{c,n}-2\alpha^{\left(i\right)}_c  \sum^K_{k=1} A\left(\boldsymbol{a}_{k}, {\boldsymbol{S}_{k}}\right)\nabla_{\boldsymbol{\theta}_{c,n}} \left(\gamma  \tilde{V}_{\boldsymbol{\theta}_{c,n}}\left(\boldsymbol{s}_{n,k+1}\right)-\tilde{V}_{\boldsymbol{\theta}_{c,n}}\left(\boldsymbol{s}_{n,k}\right) \right)$.
%\vspace{0.5pt}  
\STATE Calculate individual advantages $\tilde {A}_n\left({a}_{n,k}, {\boldsymbol{s}_{n,k}}\right)$.
\vspace{2pt}  
\STATE Update policy function with (\ref{eq:pupdate}). %$\boldsymbol{\theta}^{\left(i+1\right)}_{a,n} =\boldsymbol{\theta}^{\left(i\right)}_{a,n} + \alpha^{\left(i\right)}_a \sum^K_{k=1}\tilde {A}_n\left({a}_{n,k}, {\boldsymbol{s}_{n,k}}\right)\nabla_{\boldsymbol{\theta}_{a,n}}\log \pi_{\boldsymbol{\theta}_{a,n}}\left(a_{n,k}\left| {\boldsymbol{s}_{n,k}} \right.\right)$.
\vspace{2pt}  
\ENDFOR
\vspace{2pt}  
\ENDFOR
%\vspace{2pt}  
%\RETURN{ Optimal allocation vectors  $\boldsymbol{\theta}^*=\boldsymbol{\rho}^*$, $\boldsymbol{\varphi}^*=\boldsymbol{\delta}^*$ and $\boldsymbol{\varepsilon}^*=\boldsymbol{\beta}^*$}. 
\end{algorithmic}
\end{algorithm} 
The proposed VD-RL solution is summarized in Algorithm \ref{alg:VD-RL}.
At the beginning of the algorithm, each DBS serves its users with initial value function $V_{\boldsymbol{\theta}^{\left(1\right)}_{c,n}}$ and policy function $\boldsymbol{\pi}_{\boldsymbol{\theta}^{\left(1\right)}_{a,n}}$. The successful service rates that the DBSs achieved with the selected policy functions are calculated and recorded by the DBSs upon their return to the origin. 
The DBSs update their value parameters and policy parameters based on (\ref{eq:vupdate}) and (\ref{eq:pupdate}), using the recorded successful service rates.  Then, the DBSs serve the users with the updated policy functions.
% This VD-RL training procedure is executed by updating policy functions based on the DBSs's experiences on serving users. In essence, the DBSs update their policy episodically with a formerly generated experience at the origin.
 \vspace{-0.3cm}
 \subsubsection{ Convergence and complexity analysis} Next, we show that the proposed VD-RL algorithm is guaranteed to converge to a local team optimal strategy. 
 \begin{proposition}\label{proposition2}
\emph{The proposed distributed VD-RL algorithm is guaranteed to converge to a local optimal solution of problem (\ref{opt}), if the following conditions are satisfied:
}
\emph{\begin{enumerate}
	\item Value function $\tilde{V}_{\boldsymbol{\theta}_{c,n}}\left(\boldsymbol{s}_{n,k}\right)$ converges to a local minimum with $\nabla_{\boldsymbol{\theta}_{c,n}} A^2\left(\boldsymbol{a}_{k}, {\boldsymbol{S}_{k}}\right)=0$, and advantage function $A\left(\boldsymbol{a}_{k}, {\boldsymbol{S}_{k}}\right)\neq0$.
	\item $\mathop {\max }\limits_{a_{n,k}\in{\mathcal{C}},\boldsymbol{s}_{n,k}\in{\mathcal{S}} }\left|\frac{\partial \prod\limits_{n\in\mathcal{N}}\pi_{\boldsymbol{\theta}_{a,n}}\left(a_{n,k}\left| {\boldsymbol{s}_{n,k}} \right.\right)}{\partial \theta_{i} \partial \theta'_{j}}\right|<\infty$, for any elements $\theta_{i}$ and $\theta_{j}$ in the policy function parameter vector $\boldsymbol{\theta}_{a}=\left[\boldsymbol{\theta}_{a,1}, \boldsymbol{\theta}_{a,2}, \ldots, \boldsymbol{\theta}_{a,N}\right]$.
	\item $\lim_{i \to \infty}\alpha^{\left(i\right)}_c=0$, $\sum^{\infty}_{i=1}\alpha^{\left(i\right)}_c=\infty$, $\lim_{i \to \infty}\alpha^{\left(i\right)}_a=0$ and $\sum^{\infty}_{i=1}\alpha^{\left(i\right)}_a=\infty$. 
\end{enumerate}}
\end{proposition} 

\begin{proof} See Appendix A.
 \end{proof}
From Proposition 1, we can see that the proposed distributed VD-RL algorithm is guaranteed to converge to a local optimal solution of problem (\ref{opt}), by decomposing value $V\left(\boldsymbol{S}_k\right)$ to each DBSs. %depends on the convergence on the individual value functions $\tilde{V}_{\boldsymbol{\theta}_{c,n}}\left(\boldsymbol{s}_{n,k}\right)$.

With regards to the complexity of the proposed VD-RL algorithm, one can see that only the successful service rate and estimated future reward, i.e. individual value, reached by each DBS at each step on their trajectories will be shared and transmitted among the DBSs. Thus, the VD-RL algorithm reduces the dimensionality of the problem by updating the policy and value functions at each DBS based only on this DBS's actions and states. Note that, the studied multi-agent problem is much more complex than a single agent problem, as the exponential growth on state space and the action space in multi-agent problem causes a curse of dimensionality\footnote{Curse of dimensionality means that the neural networks' estimation error increases when they deal with high dimensional data \cite{bellman1966dynamic}}. 
Nonetheless, our approach allows to reduce this multi-agent dimensionality to that of a single-agent problem. In essence, the complexity of the VD-RL solution is $\mathcal{O}\left(\upsilon \left(n_c+n_a\right) C \right)$, where $\upsilon$ is the iteration when the proposed VD-RL algorithm converged, while $n_c$ and $n_a$ are the number of elements in $\boldsymbol{\theta}_{c,n}$ and $\boldsymbol{\theta}_{a,n}$.  $C$ is the time complexity of calculating the gradient of each element in $\boldsymbol{\theta}_{c,n}$ and $\boldsymbol{\theta}_{a,n}$. This complexity is considerably low since, as already mentioned, it is similar to that of a single agent policy gradient algorithm. 

In summary,  using the VD-RL algorithm, the DBSs update their own policy and value functions and can finally reach a local optimal solution of problem (\ref{opt}). However, when the environment changes, the local optimal strategy reached by the VD-RL algorithm will no longer be the strategy that maximizes coverage. In this case, the whole procedure in Algorithm \ref{alg:VD-RL} must be performed again to solve problem (\ref{opt}). As the studied wireless environment is highly dynamic, the VD-RL algorithm must be repeatedly performed to search a team optimal strategy in every environment. This increases the complexity of finding a VD-RL solution to a time complexity $\mathcal{O}\left(\upsilon \left(n_c+n_a\right) C X\right)$, when the user access requests change $X$ times. In order to reduce this time complexity, we use the framework of meta-learning \cite{finn2017model}. In particular, we propose, a meta training procedure that can supplement the proposed algorithm by finding an initialization of the policy and value functions for the VD-RL solution. The meta-trained initialization are close to optimal policy and value functions at all possible environments which enables the VD-RL algorithm to quickly converge in dynamic environments.

\vspace{-0.3cm}
 \subsection{Meta Training Procedure}

  \begin{algorithm}[t]\footnotesize

\caption{Meta training procedure for the trajectory design problem. }

\label{alg:VD-MAMRL}   
\setlength{\abovecaptionskip}{-40pt} 
\setlength{\belowcaptionskip}{-40pt}
\begin{algorithmic} [1] %è¿ä¸ª1 è¡¨ç¤ºæ¯ä¸è¡é½æ¾ç¤ºæ°å­  
\REQUIRE A distribution of user requests $p\left(\mathcal{Z}\right)$, user locations, time constraints. \\ 
\vspace{2pt}  
\ENSURE Initialize value functions $\tilde{V}_{\boldsymbol{\theta}^{\left(1\right)}_{c,n}}$, and policy functions $\boldsymbol{\pi}_{\boldsymbol{\theta}^{\left(1\right)}_{a,n}}$, for $n \in\mathcal{N}$.\\ %, with, respectively, parameter $\boldsymbol{\theta}^{\left(1\right)}_{c,n}$ and $\boldsymbol{\theta}^{\left(1\right)}_{a,n}$. \\
\vspace{2pt}  
\FOR {Meta training epoch $i=1:I$} 
\vspace{2pt}  
\FOR {Environment sampling epoch $j=1:J$} 
\vspace{2pt}  
\STATE Sample an user request realization $\boldsymbol{z}_j\sim\mathcal{Z}$.
\vspace{2pt}  
\STATE Carry out an experience by generating a sequence of actions that are randomly selected based on current policy functions, i.e. $\boldsymbol{\pi}_{\boldsymbol{\theta}^{\left(i\right)}_{a,n}}$, for $n\in\mathcal{N}$.
\vspace{2pt}  
\STATE Calculate advantage $A$ based on (\ref{eq:tderror}).
\vspace{2pt}  
\FOR {Each DBS $n=1:N$} 
\vspace{2pt}  
\STATE Update individual value functions $\boldsymbol{\theta}^{\left(i\right)'}_{c,n,j}=\boldsymbol{\theta}^{\left(i\right)}_{c,n} + \alpha_c\nabla_{\boldsymbol{\theta}^{\left(i\right)}_{c,n}} \sum^K_{k=1} {A^2\left(\boldsymbol{a}^{\left(\boldsymbol{e}_{n,j}\right)}_{k}, {\boldsymbol{S}^{\left(\boldsymbol{e}_{n,j}\right)}_{k}}\right)}$ and policy functions $\boldsymbol{\theta}^{\left(i\right)'}_{a,n,j}=\boldsymbol{\theta}^{\left(i\right)}_{a,n}+ \alpha_a \sum^K_{k=1}\tilde {A}_n\left({a}^{\left(\boldsymbol{e}_{n,j}\right)}_{n,k}, {\boldsymbol{s}^{\left(\boldsymbol{e}_{n,j}\right)}_{n,k}}\right)\nabla_{\boldsymbol{\theta}^{\left(i\right)}_{a,n}}\log \pi_{\boldsymbol{\theta}^{\left(i\right)}_{a,n}}\left(a^{\left(\boldsymbol{e}_{n,j}\right)}_{n,k}\left| {\boldsymbol{s}^{\left(\boldsymbol{e}_{n,j}\right)}_{n,k}} \right.\right)$.
\vspace{2pt}  
\STATE Carry out an experience by generating a sequence of actions that are randomly selected based on updated policy functions, i.e. $\boldsymbol{\pi}_{\boldsymbol{\theta}^{\left(i\right)'}_{a,n,j}}$.
\vspace{2pt}  
\STATE Calculate loss $\tilde {L}_{c,n}\left(\tilde{V}_{\boldsymbol{\theta}^{\left(i\right)'}_{c,n,j}}, {\boldsymbol{z}_{j}}\right)$ and $\tilde {L}_{a,n}\left(\boldsymbol{\pi}_{\boldsymbol{\theta}^{\left(i\right)'}_{a,n,j}}, {\boldsymbol{z}_{j}}\right)$ with (\ref{eq:closs}) and (\ref{eq:ploss}).
\vspace{2pt}  
\ENDFOR
\vspace{2pt}  
\ENDFOR
\vspace{2pt} 
\FOR {Each DBS $n=1:N$} 
\vspace{2pt}  
\STATE Update initial value parameters $\boldsymbol{\theta}^{\left(i+1\right)}_{c,n} = \boldsymbol{\theta}^{\left(i\right)}_{c,n}-\beta\nabla_{\boldsymbol{\theta}^{\left(i\right)}_{c,n}}\sum_{\boldsymbol{z}_j \sim p\left(\mathcal{Z}\right)}\tilde {L}_{c,n}\left(\tilde{V}_{\boldsymbol{\theta}^{\left(i\right)'}_{c,n,j}}, {\boldsymbol{z}_{j}}\right)$, and policy parameters $\boldsymbol{\theta}^{\left(i+1\right)}_{a,n} = \boldsymbol{\theta}^{\left(i\right)}_{a,n}-\beta\nabla_{\boldsymbol{\theta}^{\left(i\right)}_{a,n}}\sum_{\boldsymbol{z}_j \sim p\left(\mathcal{Z}\right)}\tilde {L}_{a,n}\left(\boldsymbol{\pi}_{\boldsymbol{\theta}^{\left(i\right)'}_{a,n,j}}, {\boldsymbol{z}_{j}}\right)$.
\ENDFOR
\vspace{2pt}  
\ENDFOR
\vspace{2pt}  
\RETURN{ Optimal value functions $\tilde{V}_{\boldsymbol{\theta}^{*}_{c,n}}$, and policy functions $\boldsymbol{\pi}_{\boldsymbol{\theta}^{*}_{a,n}}$, for $n\in\mathcal{N}$}. 
\end{algorithmic}
\end{algorithm} 

Next, we introduce our meta-learning approach that meta trains a VD-RL solution using model agnostic meta-learning (MAML)\cite{finn2017model}. The proposed meta training method seeks a VD-RL solution capable of quickly reaching team optimal strategies in various environments. This meta training procedure prepares the VD-RL solution with a set of well established initial policy and value functions with suitable estimation on a distribution of user request realizations, i.e. $p\left(\mathcal{Z}\right)$.  These initial policy and value functions are close to the optimal ones that provide team optimal strategies to each of the user request realizations in $\mathcal{Z}$.
During the meta training procedure, a realization $\boldsymbol{z}_j$ is first sampled from $p\left(\mathcal{Z}\right)$. Then, the DBSs collect experiences $\boldsymbol{e}_{n,j}$ using their initial policy functions $\boldsymbol{\pi}_{\boldsymbol{\theta}^{\left(1\right)}_{a,n}}$, with the consideration that ground users are requesting access based on $\boldsymbol{z}_j$. The policy and value functions at the DBSs are updated using Algorithm \ref{alg:VD-RL}, based on the collected experiences $\boldsymbol{e}_{n,j}$. The updated policy and value functions at DBS $n$ within $\boldsymbol{z}_j$ are denoted as $\tilde{V}_{\boldsymbol{\theta}'_{c,n,j}}$, and $\boldsymbol{\pi}_{\boldsymbol{\theta}'_{a,n,j}}$, respectively. \iffalse Here, $\boldsymbol{\theta}'_{c,n,j}=\boldsymbol{\theta}_{c,n} + \alpha_c\nabla_{\boldsymbol{\theta}_{c,n}} \sum^K_{k=1} {A^2\left(\boldsymbol{a}_{k}, {\boldsymbol{S}_{k}}\right)}$ and $\boldsymbol{\theta}'_{a,n,j}=\boldsymbol{\theta}_{a,n}+ \alpha_a \sum^K_{k=1}\tilde {A}_n\left({a}_{n,k}, {\boldsymbol{s}_{n,k}}\right)\nabla_{\boldsymbol{\theta}_{a,n}}\log \pi_{\boldsymbol{\theta}_{a,n}}\left(a_{n,k}\left| {\boldsymbol{s}_{n,k}} \right.\right)$ are, respectively, the updated and value parameters.\fi The updated policy and value functions are then tested on new experiences $\boldsymbol{e}'_{n,j}$ sampled using $\boldsymbol{\pi}_{\boldsymbol{\theta}'_{a,n,j}}$ at current realization $\boldsymbol{z}_j$. The feedback from such tests at DBS $n$ will be the values of loss functions that measure the distance from the updated policy and value functions to the optimal policy and value functions that produce team optimal strategies. These are given by
  \begin{equation}\label{eq:closs}
  \setlength{\abovedisplayskip}{3 pt}
\setlength{\belowdisplayskip}{3 pt}
\begin{split}
\tilde {L}_{c,n}\left(\tilde{V}_{\boldsymbol{\theta}'_{c,n,j}}, {\boldsymbol{z}_{j}}\right)= \sum^K_{k=1}\left(r\left(\boldsymbol{a}^{\left(\boldsymbol{e}'_{n,j}\right)}_k\left|\boldsymbol{S}^{\left(\boldsymbol{e}'_{n,j}\right)}_k\right.\right) + \gamma  \tilde{V}_{\boldsymbol{\theta}'_{c,n,j}}\left(\boldsymbol{s}^{\left(\boldsymbol{e}'_{n,j}\right)}_{n,k+1}\right)- \tilde{V}_{\boldsymbol{\theta}'_{c,n,j}}\left(\boldsymbol{s}^{\left(\boldsymbol{e}'_{n,j}\right)}_{n,k}\right)\right)^2,
\end{split}
\end{equation}
 \begin{equation}\label{eq:ploss}
\begin{split}
%\pi_{\boldsymbol{\theta}_c}\left(a_{n,k}\left| {\boldsymbol{s}_{n,k}} \right.\right) = \pi_{\boldsymbol{\theta}_c}\left(a_{n,k}\left| {\boldsymbol{s}_{n,k}} \right.\right)+\alpha_a\sum^K_{k=1}\tilde {A}_n\left({a}_{n,k}, {\boldsymbol{s}_{n,k}}\right)\nabla_{\boldsymbol{\theta}_c}\log \pi_{\boldsymbol{\theta}_c}\left(a_{n,k}\left| {\boldsymbol{s}_{n,k}} \right.\right),
\tilde {L}_{a,n}\left(\boldsymbol{\pi}_{\boldsymbol{\theta}'_{a,n,j}}, {\boldsymbol{z}_{j}}\right)=\sum^K_{k=1}&\left(r\left(\boldsymbol{a}^{\left(\boldsymbol{e}'_{n,j}\right)}_k\left|\boldsymbol{S}^{\left(\boldsymbol{e}'_{n,j}\right)}_k\right.\right)+ \gamma \tilde{V}_{\boldsymbol{\theta}'_{c,n,j}}\left(\boldsymbol{s}^{\left(\boldsymbol{e}'_{n,j}\right)}_{n,k+1}\right)- \tilde{V}_{\boldsymbol{\theta}'_{c,n,j}}\left(\boldsymbol{s}^{\left(\boldsymbol{e}'_{n,j}\right)}_{n,k}\right)\right)\\
&\log \pi_{\boldsymbol{\theta}'_{a,n,j}}\left(a^{\left(\boldsymbol{e}'_{n,j}\right)}_{n,k}\left| {\boldsymbol{s}^{\left(\boldsymbol{e}'_{n,j}\right)}_{n,k}} \right.\right),
\end{split}
\end{equation}
where $\boldsymbol{a}^{\left(\boldsymbol{e}'_{n,j}\right)}_k$ and $\boldsymbol{S}^{\left(\boldsymbol{e}'_{n,j}\right)}_k$ are, respectively, the $k$-th action and state of DBS $n$ within experience $\boldsymbol{e}'_{n,j}$. By minimizing the loss functions in (\ref{eq:closs}) and (\ref{eq:ploss}), 
the proposed meta training method updates the policy and value function parameters toward the parameters of optimal policy and value functions in various user access request realizations sampled from $p\left(\mathcal{Z}\right)$. More concretely, the objective of the meta training procedure is
 \begin{equation}\label{eq:metaobj}
\begin{split}
%\pi_{\boldsymbol{\theta}_c}\left(a_{n,k}\left| {\boldsymbol{s}_{n,k}} \right.\right) = \pi_{\boldsymbol{\theta}_c}\left(a_{n,k}\left| {\boldsymbol{s}_{n,k}} \right.\right)+\alpha_a\sum^K_{k=1}\tilde {A}_n\left({a}_{n,k}, {\boldsymbol{s}_{n,k}}\right)\nabla_{\boldsymbol{\theta}_c}\log \pi_{\boldsymbol{\theta}_c}\left(a_{n,k}\left| {\boldsymbol{s}_{n,k}} \right.\right),
\min_{\boldsymbol{\theta}_{c}, \boldsymbol{\theta}_{a}}&\sum_{\boldsymbol{z}_j \sim p\left(\mathcal{Z}\right)}\sum^N_{n=1}\tilde {L}_{c,n}\left(\tilde{V}_{\boldsymbol{\theta}'_{c,n,j}}, {\boldsymbol{z}_{j}}\right)+\tilde {L}_{a,n}\left(\boldsymbol{\pi}_{\boldsymbol{\theta}'_{a,n,j}}, {\boldsymbol{z}_{j}}\right),
%\\
%&=\min_{\boldsymbol{\theta}_{c,n}, \boldsymbol{\theta}_{a,n}}\sum_{\boldsymbol{z}_j \sim p\left(\mathcal{Z}\right)}\sum^N_{n=1}\tilde {L}_{c,n}\left(\tilde{V}_{\boldsymbol{\theta}_{c,n} + \alpha_c\nabla_{\boldsymbol{\theta}_{c,n}} \sum^K_{k=1} {A^2\left(\boldsymbol{a}^{\left(\boldsymbol{e}_{n,j}\right)}_{k}, {\boldsymbol{S}^{\left(\boldsymbol{e}_{n,j}\right)}_{k}}\right)}}, {\boldsymbol{z}_{j}}\right)\\
%&+\tilde {L}_{a,n}\left(\boldsymbol{\pi}_{\boldsymbol{\theta}_{a,n}+ \alpha_a \sum^K_{k=1}\tilde {A}_n\left({a}^{\left(\boldsymbol{e}_{n,j}\right)}_{n,k}, {\boldsymbol{s}^{\left(\boldsymbol{e}_{n,j}\right)}_{n,k}}\right)\nabla_{\boldsymbol{\theta}_{a,n}}\log \pi_{\boldsymbol{\theta}_{a,n}}\left(a^{\left(\boldsymbol{e}_{n,j}\right)}_{n,k}\left| {\boldsymbol{s}^{\left(\boldsymbol{e}_{n,j}\right)}_{n,k}} \right.\right)}, {\boldsymbol{z}_{j}}\right).
\end{split}
\end{equation}
where $\boldsymbol{\theta}_{c}=\left[\boldsymbol{\theta}_{c,1}, \boldsymbol{\theta}_{c,2},\ldots, \boldsymbol{\theta}_{c,N}\right]$ is the vector of value function parameters at all DBSs. {\small$\boldsymbol{\theta}'_{c,n,j}=\boldsymbol{\theta}_{c,n} + \alpha_c\nabla_{\boldsymbol{\theta}_{c,n}} \sum^K_{k=1} {A^2\left(\boldsymbol{a}^{\left(\boldsymbol{e}_{n,j}\right)}_{k}, {\boldsymbol{S}^{\left(\boldsymbol{e}_{n,j}\right)}_{k}}\right)}$} is the updated value function parameters of DBS $n$, based on experience $\boldsymbol{e}_{n,j}$. {\small$\boldsymbol{\theta}'_{a,n,j}={\theta _{a,n}} + {\alpha _a}\sum\limits_{k = 1}^K {{{\tilde A}_n}} \left( {a_{n,k}^{\left( {{{\boldsymbol{e}}_{n,j}}} \right)},{\boldsymbol{s}}_{n,k}^{\left( {{{\boldsymbol{e}}_{n,j}}} \right)}} \right){\nabla _{{\theta _{a,n}}}}\log {\pi _{{\theta _{a,n}}}}\left( {a_{n,k}^{\left( {{{\boldsymbol{e}}_{n,j}}} \right)}\left| {{\boldsymbol{s}}_{n,k}^{\left( {{{\boldsymbol{e}}_{n,j}}} \right)}} \right.} \right)$} is the updated policy function parameter of DBS $n$, based on experience $\boldsymbol{e}_{n,j}$.
Here, we can see that the optimization variables in (\ref{eq:metaobj}) are initial policy and value function parameters $\boldsymbol{\theta}_{c}$ and $\boldsymbol{\theta}_{a}$.  That is, the meta-learning approach seeks to find the optimal initial function parameters of the VD-RL algorithm so as to minimize the estimation errors captured in (\ref{eq:closs}) and (\ref{eq:ploss}),  for different environments that can be sampled from $p\left(\mathcal{Z}\right)$. This parameter initialization, i.e. the meta-trained policy and value functions, will make it easier for the VD-RL solution to find the team optimal strategies in unseen environments.
%can converge to team optimal strategies in unseen environments with only a small number of gradient steps. 

\begin{figure}
  \centering
  \includegraphics[width=9 cm]{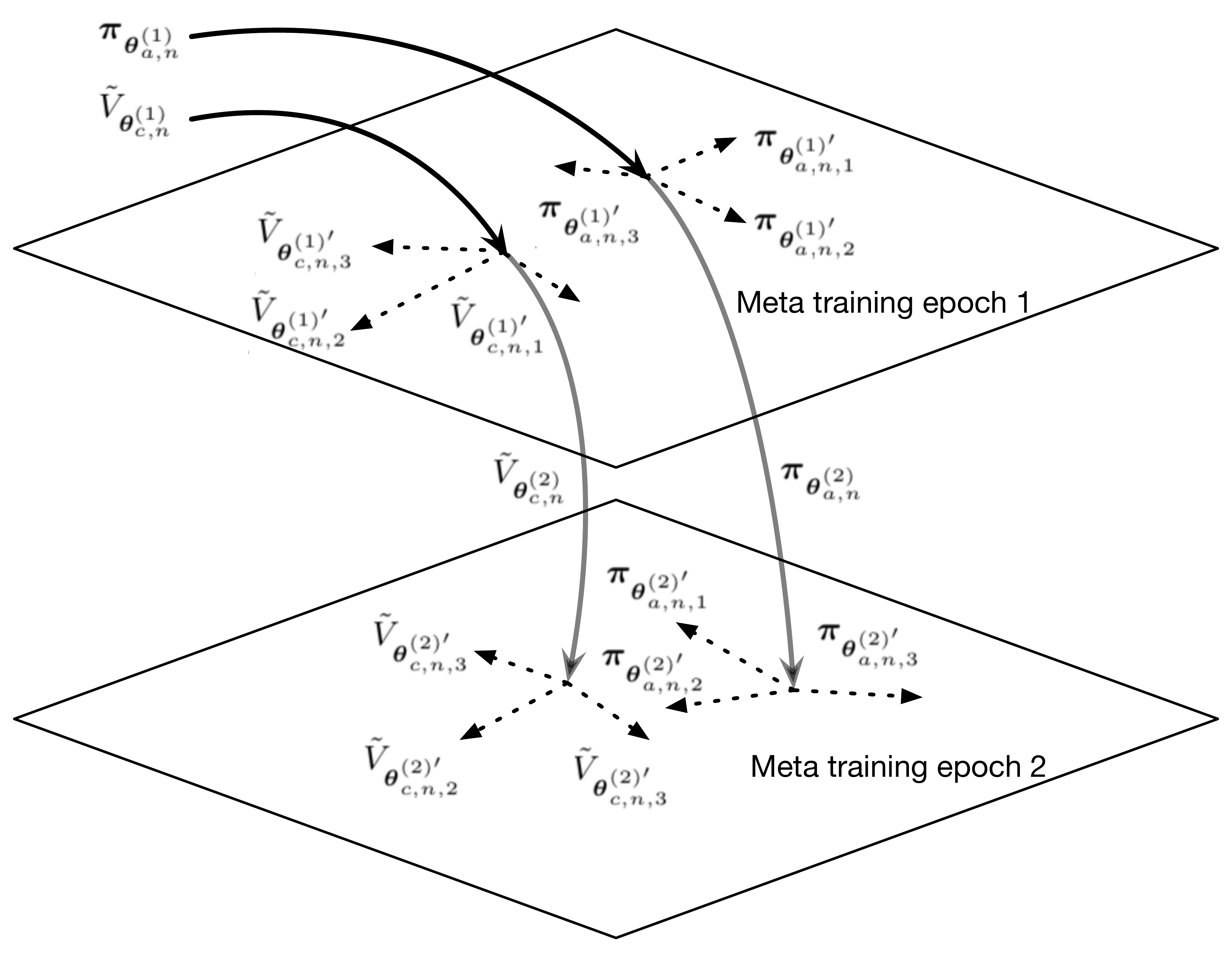}
  \caption{\footnotesize{Illustration of the proposed meta-training method at DBS $n$ with $J=3$. That is, $3$ different user requests are drawn from $p\left(\mathcal{Z}\right)$ at each training episode. $\tilde{V}_{\boldsymbol{\theta}^{\left(1\right)'}_{c,n,1}}$ and $\boldsymbol{\pi}_{\boldsymbol{\theta}^{\left(1\right)'}_{a,n,1}}$ are the updated policy and value functions in the first sampled environment,  }}
  \label{Fig. 3}
  \centering
  \vspace{-0.7cm}
\end{figure}

To solve problem (\ref{eq:metaobj}), a standard stochastic gradient descent method is applied. Fig. \ref{Fig. 3} shows one iteration for solving problem (\ref{eq:metaobj}). For this example, at each meta training iteration, the meta training method performs one step stochastic gradient descent on (\ref{eq:metaobj}) over the initial policy and value function parameters $\boldsymbol{\theta}_{c,n}$ and $\boldsymbol{\theta}_{a,n}$ based on sampled experiences in $J$ different environments. The full procedure of using meta training method with VD-RL is summarized in Algorithm \ref{alg:VD-MAMRL}. In this meta training procedure, the DBSs collect their experiences from serving $J$ different user requests drawn from $p\left(\mathcal{Z}\right)$, with initial value functions $\tilde{V}_{\boldsymbol{\theta}^{\left(1\right)}_{c,n}}$ and policy functions $\boldsymbol{\pi}_{\boldsymbol{\theta}^{\left(1\right)}_{a,n}}$. After serving one sampled user request, the DBSs execute one step VD-RL update and get new value functions $\tilde{V}_{\boldsymbol{\theta}^{\left(1\right)'}_{c,n,j}}$ and policy functions $\boldsymbol{\pi}_{\boldsymbol{\theta}^{\left(1\right)'}_{a,n,j}}$. The updated policy and value functions are evaluated by the loss functions defined in (\ref{eq:closs}) and (\ref{eq:ploss}), based on experiences collected with strategy $\boldsymbol{\pi}_{\boldsymbol{\theta}^{\left(1\right)'}_{a,n,j}}$. The meta training procedure in Algorithm \ref{alg:VD-MAMRL} seeks to finding a set of initial policy and value functions that are close to the optimal policy and value functions at every user request realization. Equipped with this initialization, the VD-RL can now find a local optimal solution of problem (\ref{opt}) in unseen environment using a small number of update steps in Algorithm \ref{alg:VD-RL}. Thus, the meta-training method minimizes the losses that are collected from the sampled user request realizations, by updating the value and policy parameters, i.e. $\boldsymbol{\theta}^{\left(1\right)}_{c,n}$, $\boldsymbol{\theta}^{\left(1\right)}_{a,n}$, in the opposite direction of the gradient of the losses. After one step of this update, the value and policy parameters, i.e. $\boldsymbol{\theta}^{\left(2\right)}_{c,n}$, $\boldsymbol{\theta}^{\left(2\right)}_{a,n}$, will be closer to the optimal policy and value functions in the sampled environment. Subsequently, the DBSs will start serving user requests with value functions $\tilde{V}_{\boldsymbol{\theta}^{\left(2\right)}_{c,n}}$, and policy functions $\boldsymbol{\pi}_{\boldsymbol{\theta}^{\left(2\right)}_{a,n}}$. Then, the DBSs collect their experience from serving different user requests, based on which they can keep updating the policy and value functions, as shown in Fig. \ref{Fig. 3}, until convergence. Using our proposed meta training procedure, the DBSs finally find a VD-RL solution with optimal initializations of the policy and value function parameters, i.e. $\tilde{V}_{\boldsymbol{\theta}^{*}_{c,n}}$, and $\boldsymbol{\pi}_{\boldsymbol{\theta}^{*}_{a,n}}$. %These initializations provide a set of policy and value functions that locates close to the optimal policy and value functions of all possible user request realizations in $p\left(\mathcal{Z}\right)$. Thus, the DBSs can find their team optimal strategies to serve all possible, including unseen, user request realizations in $p\left(\mathcal{Z}\right)$, with a small number of VD-RL iterations.
 The complexity of the meta training procedure is $\mathcal{O}\left(2\upsilon_m \left(n_c+n_a\right) CJ \right)$, with $\upsilon_m$ being the number of iterations that the meta training method needs for convergence. This complexity is considerably low, as the proposed meta training mechanism does not needs any additional neural networks in the training procedure. Also, since the DBSs can collect experiences on serving group users' daily requests, the proposed meta training procedure's complexity will not introduce additional costs in terms of time, energy, or DBS hardware.   %$n_c$ and $n_a$ are the number of elements within value parameters $\boldsymbol{\theta}_{c,n}$ and $\boldsymbol{\theta}_{a,n}$,  $C$ is the time complexity of calculating the gradient of each element (i.e. each element in $\boldsymbol{\theta}_{c,n}$ and $\boldsymbol{\theta}_{a,n}$).

 \vspace{-0.3cm}
\section{Simulation Results and Analysis} 
{
For our simulations, we consider a scenario with five DBSs serving $U = 300$ mobile users. 
%Note that, the height of the DBSs is assumed to be $200$ meters. The users are repeatedly deployed at some random locations that are uniformly distributed in an area with a diameter of $15000$ meters. 
The main simulation parameters are listed in Table I. The results of the proposed VD-RL algorithm are compared with the independent actor critic algorithm \cite{foerster2017counterfactual} and the monotonic value function factorisation algorithm \cite{rashid2018qmix}. The independent actor critic algorithm  (IAC) updates one policy function and one value function for all of the agents. The monotonic value function factorisation algorithm, called Q mix, uses the same value and policy update procedure as our proposed VD-RL algorithm, but it estimates the summation in (\ref{eq:vd}) using a neural network.
 The results of the proposed meta-trained VD-RL  are further compared to the results from the pre-trained VD-RL algorithm\cite{5740583}, the original VD-RL algorithm, as well as oracle results. Within the pre-trained VD-RL algorithm, the VD-RL algorithm starts from the optimal policy and value functions at former tasks. Here, the pre-trained VD-RL algorithm trains a VD-RL within exactly the same environments sampled at our proposed meta training procedure.
For the original VD-RL algorithm, the DBSs start with randomly initialized policy and value functions. The oracle results present the DBSs' performance in the considered environment, using the team optimal strategies.  All statistical results are averaged over a large number of independent runs.

%Other parameters used in the simulations are listed in {{Table II}}. The heavy ball based Walrasian equilibrium results are compared to the centralized optimization based results, the sub-gradient based results, and a random allocation scheme results considering the market clearance. {Here, with {{the centralized optimization algorithm \cite{7742722}}, the optimization problem formed in (21) is solved at one MBS and the optimal solutions are transmitted to all of the other nodes (users, BSs, MBSs, and the satellite) in the network. {{The MATLAB optimization toolbox is used to implement the centralized algorithm.}} The sub-gradient algorithm represents the primal-dual sub-gradient algorithm in \cite{boyd2004convex}. The settings of the sub-gradient algorithm are similar to the heavy ball algorithm, although any other approach is possible. However, in the sub-gradient algorithm, the dual variables are directly updated in the direction of the current negative sub-gradient. With the random allocation scheme, the time resources are randomly allocated to the users and BSs. $2000$ independent runs are used (for averaging) in order to generate the statistical results.}  

\begin{table}
 \vspace{-0.32cm}
  \newcommand{\tabincell}[2]{\begin{tabular}{@{}#1@{}}#2\end{tabular}}
\renewcommand\arraystretch{1}
 \caption{
    \vspace*{-0.2em}Simulation Parameters\cite{7037248}}\vspace*{-1em}
\centering  
\begin{tabular}{|c|c|c|c|}% ÃÂ±ÃÂ­ÃÂÃÅ¾ÃÅ¸ÃÂ·ÃÂÃÂÃÂÃÂªÃÂÃÂÃÂ¶ÃÂÃÂÃÂ«ÃÂ·ÃÂÃÂÃÂÃÂ£ÃÂ¬left-l,right-r,center-c
\hline
\textbf{Parameter} & \textbf{Value} & \textbf{Parameter} & \textbf{Value} \\
\hline
$P_u $ & 20 dBm & $V_s$ & 30 m/s  \\
\hline
$N_0$ & -170 dBm/Hz & $B $ & 1 MHz\\
\hline
$\mu_{\textrm{LoS}}, $ & 1.6 & $\delta_{\textrm{LoS}}$ & 8.41\\
\hline
$\mu_{\textrm{NLoS}}, $ & 23 & $\delta_{\textrm{NLoS}}$ & 33.78\\
\hline
%$G^t_n\left(0\right)$ & 38 dBi & $G^t_s\left(0\right)$ & 42.1 dBi\\
%\hline

%$ \sigma ^2$ & -95 dBm & $ \\
%\hline
%$B_v$ & 1 GHz & &$h_{\min}$ & 100 m \\
%\hline
%$L$ & 1 Mbit & $P_{\max}$ & 20 W&$\delta_{S_i,n}$ & 5 Mbit/s\\
%\hline
%$\mu_\textrm{LoS}$,$\mu_\textrm{NLoS}$& 2, 2.4 &$N_x, N_s$ &4, 12&$\chi  $ & 15\\
%\hline
%  & &$\zeta_1,\zeta_2$&0.5,0.5\\ 
%\hline
% $\chi _{\sigma_\textrm{NLoS}}$ & 5.27 & $\beta, \eta $ & 2, 100&$X,Y$& 11.9, 0.13\\ 
%\hline
\end{tabular}
 \vspace{-0.5cm}
\end{table}

\begin{figure}[t]
\setlength{\abovecaptionskip}{-5pt} 
\setlength{\belowcaptionskip}{-8pt} 
  \begin{center}
   \vspace{0cm}
    \includegraphics[width=9.5cm]{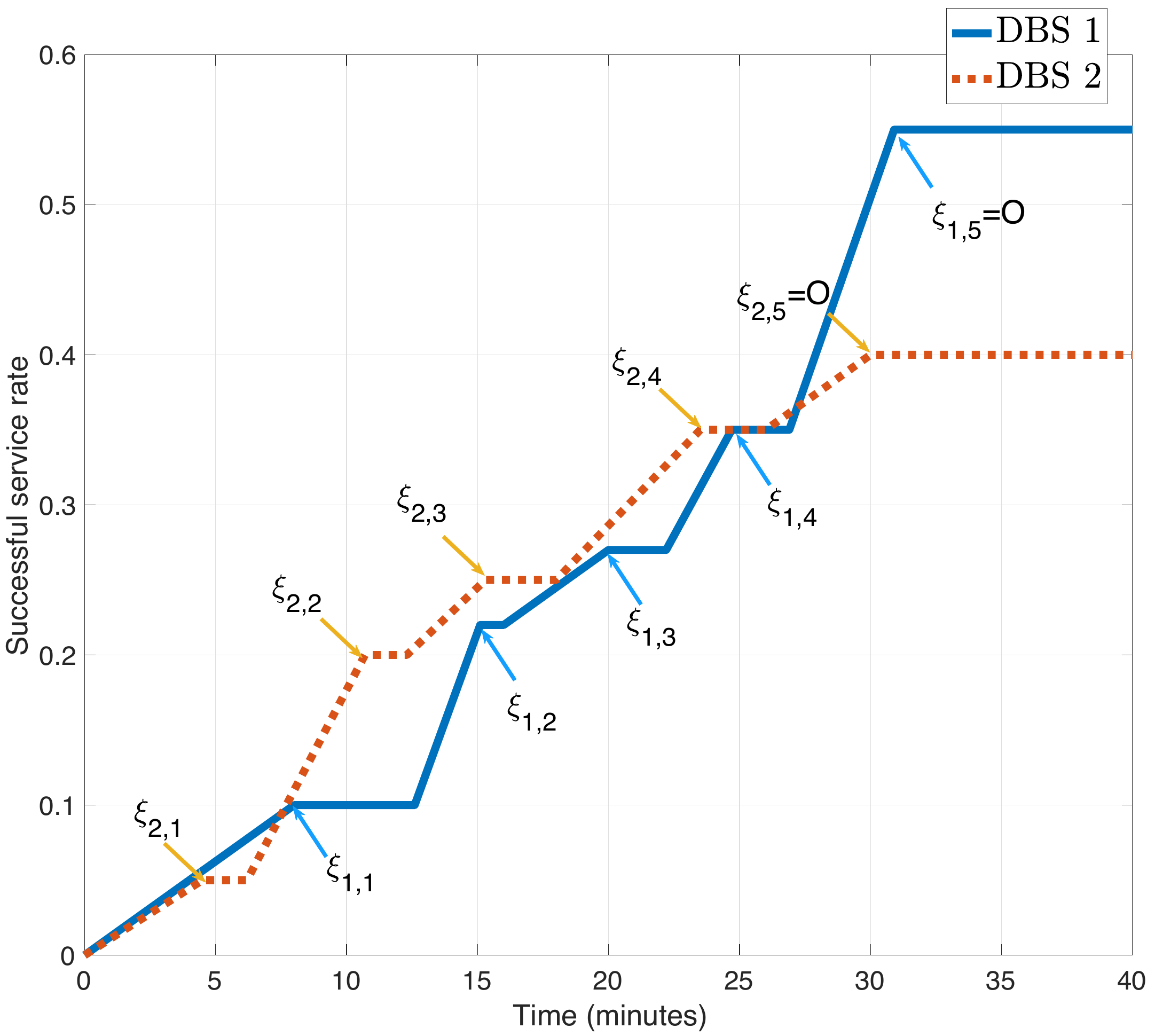}
    \vspace{-0.2cm}
    \caption{\label{VDRL4} Successful service rate achieved at each DBS over time. }
  \end{center}\vspace{-0.9cm}
\end{figure}
Fig. \ref{VDRL4} shows how successful service rates increase with time, in a network with two DBSs. From Fig. \ref{VDRL4}, we can see that the successful service rate achieved by each DBS increases when the DBS arrives at a new cluster $\xi_{n,k}$. %, and it remains the same when the DBS flies across clusters.  
Fig. \ref{VDRL4} also shows that the two considered DBSs spend different amount of time on flying across clusters and serving clusters. In other words, the DBSs operate asynchronously in the considered area. In particular, when DBS $2$ starts to serve its second cluster, DBS $1$ is still serving its first cluster.

\begin{figure}
\setlength{\belowcaptionskip}{-4pt}
\setlength{\abovecaptionskip}{-4pt}
\centering
\subfigure[]{
\label{subfig:a} %% label for first subfigure
\includegraphics[width=7.5 cm]{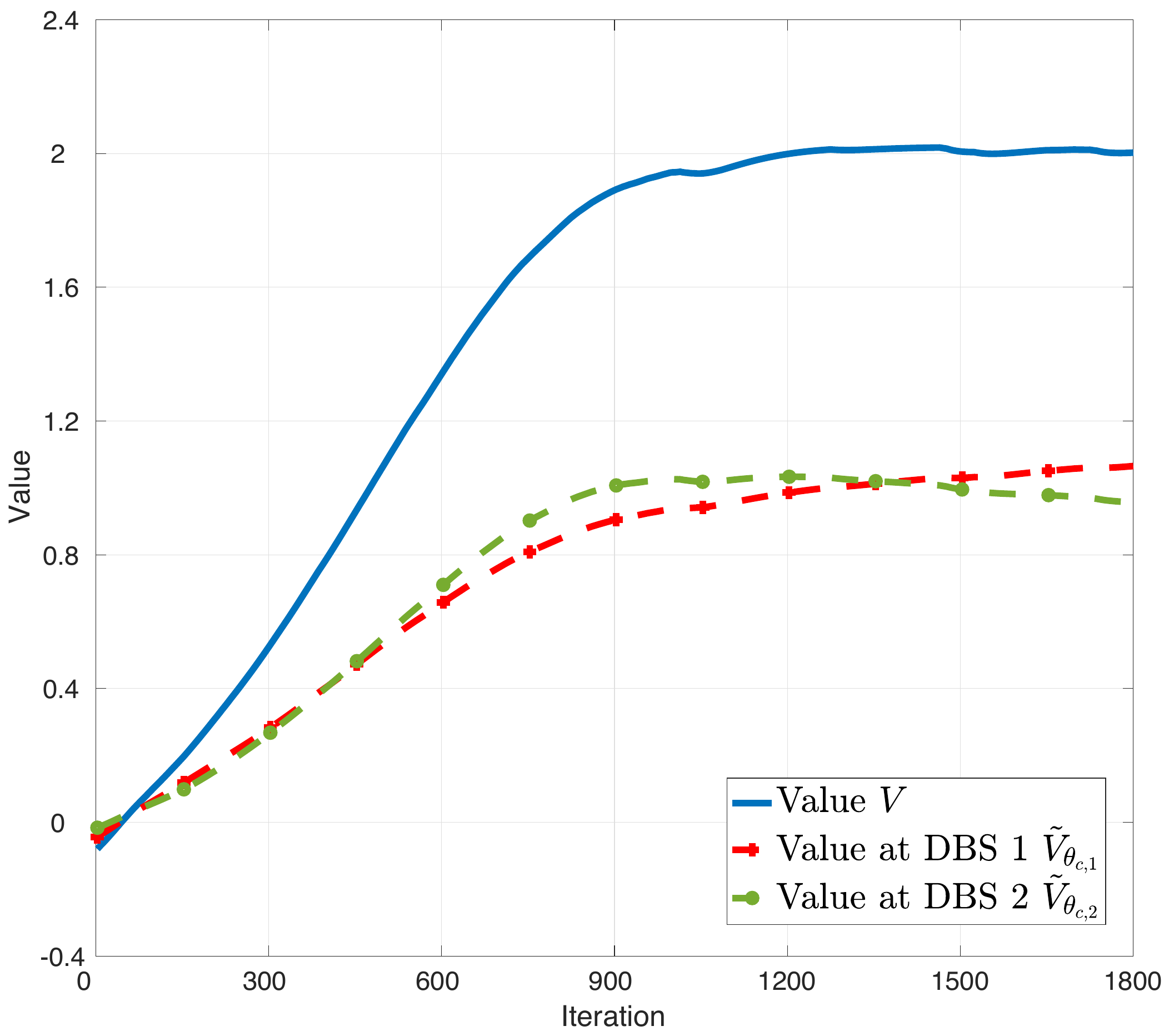}}
\subfigure[]{
\label{subfig:b} %% label for second subfigure
\includegraphics[width=7.5 cm]{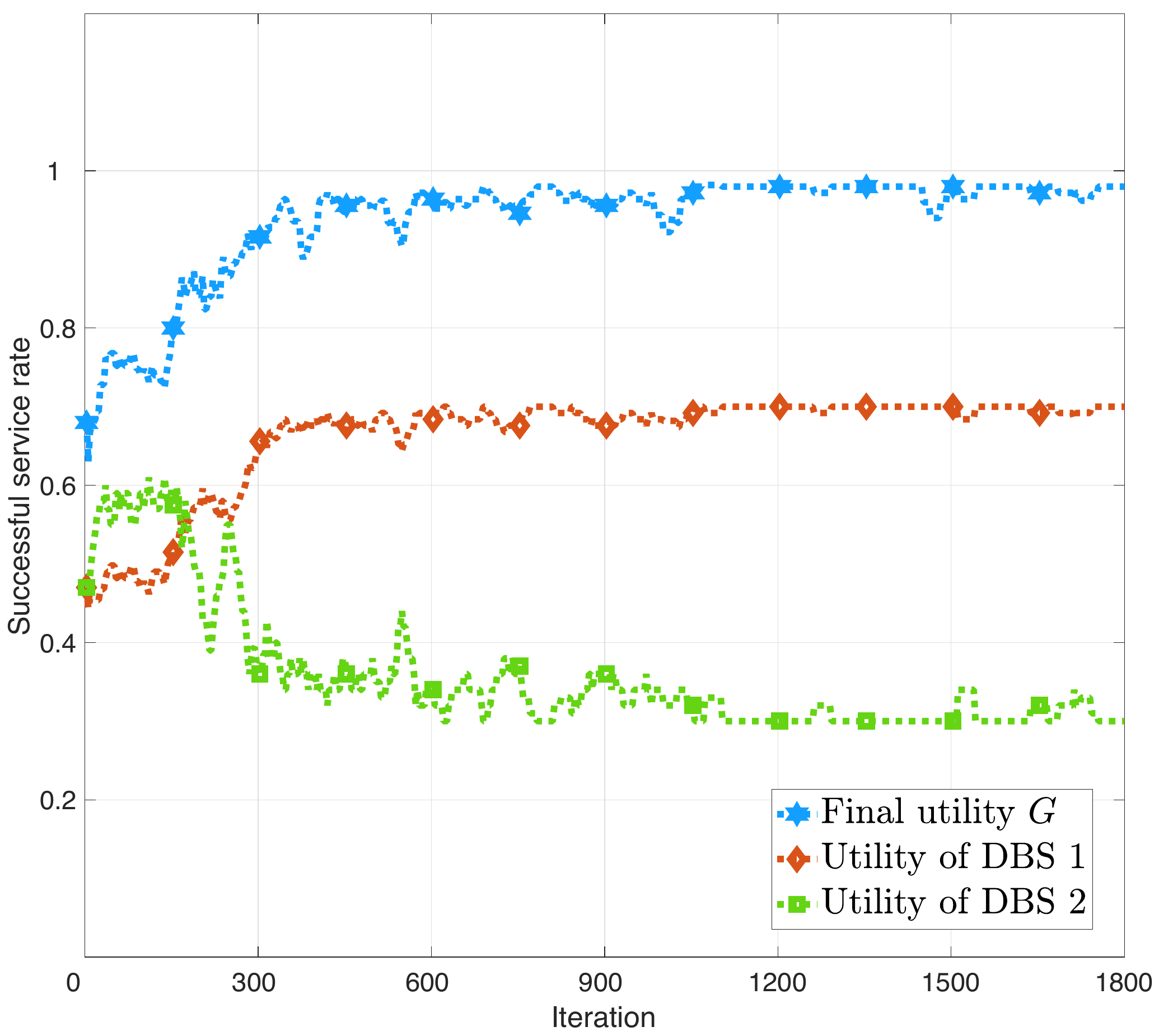}}
\caption{\footnotesize{The value decomposition learned by the proposed VD-RL algorithm.}}
\label{VDRL1} %% label for entire figure
\vspace{-0.9cm}
\end{figure}
%\begin{figure}[!t]
%\setlength{\abovecaptionskip}{0pt} 
%\setlength{\belowcaptionskip}{-15pt} 
%  \begin{center}
%   \vspace{0cm}
%    \includegraphics[width=9.5cm]{VDRL1}
%    \vspace{-0.2cm}
%    \caption{\label{VDRL1} The value decomposition learned by the proposed VD-RL algorithm.}
%  \end{center}
%  \vspace{-0.7cm} %Ã¨Â¿ÂÃ¤Å¾ÂÃ¥ÂÂ¥Ã¥Â°Â±Ã¦ÂÂ¯Ã§ÅÂ©Ã¨Â¡ÂÃ¨Â·ÂÃ§ÂÂ
%\end{figure}

Fig. \ref{VDRL1} shows how the decomposition of the value function $V\left(\boldsymbol{S}_k\right)$ is learned using the proposed VD-RL algorithm. Here, the value of function $V\left(\boldsymbol{S}_0\right)$ is the estimation of the discounted future reward at the initial state $\boldsymbol{S}_0$. 
In Fig. \ref{VDRL1}, the value of function $V\left(\boldsymbol{S}_0\right)$, as well as the individual values of functions $\tilde V_{\boldsymbol{\theta}_{c,1}}\left(\boldsymbol{S}_0\right)$ and $\tilde V_{\boldsymbol{\theta}_{c,2}}\left(\boldsymbol{S}_0\right)$ follow the same trend of as the team utility of all DBSs. This implies that the individual value functions also estimate the team utility. With such values, the DBSs can independently find the strategies that maximize the team utility without sharing information such as actions and states. In particular, at the convergence of the VD-RL algorithm, DBS 2 sacrifices its own utility for team optimality. 
From Fig. \ref{VDRL1}, we can also observe that, with the VD-RL algorithm, the individual value at DBS 1 keeps increasing, yet the successful service rate at DBS 1 increases at first and then decreases while reaching convergence. The sum of individual values at two DBSs always equals to the value of $V\left(\boldsymbol{S}_k\right)$, which follows the trend of the team utility. In other words, the VD-RL algorithm learns to decompose the value function into two individual values, each of which can direct each DBS to independently finds its strategy that maximizes the team utility, instead of its individual utility.
 \begin{figure}
\centering
\subfigure[]{
\label{subfig:a} %% label for first subfigure
\includegraphics[width=7.5 cm]{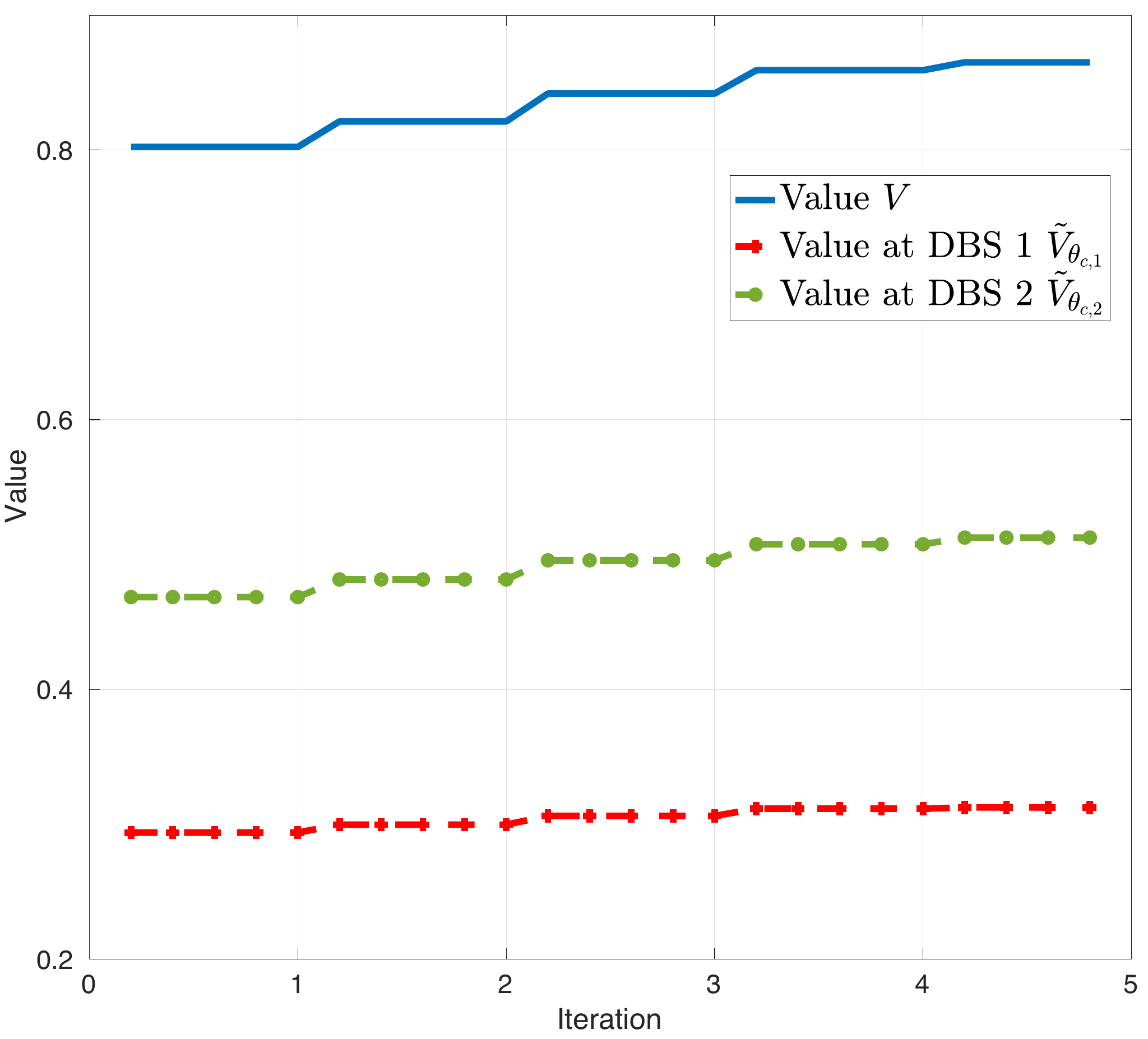}}
\subfigure[]{
\label{subfig:b} %% label for second subfigure
\includegraphics[width=7.5 cm]{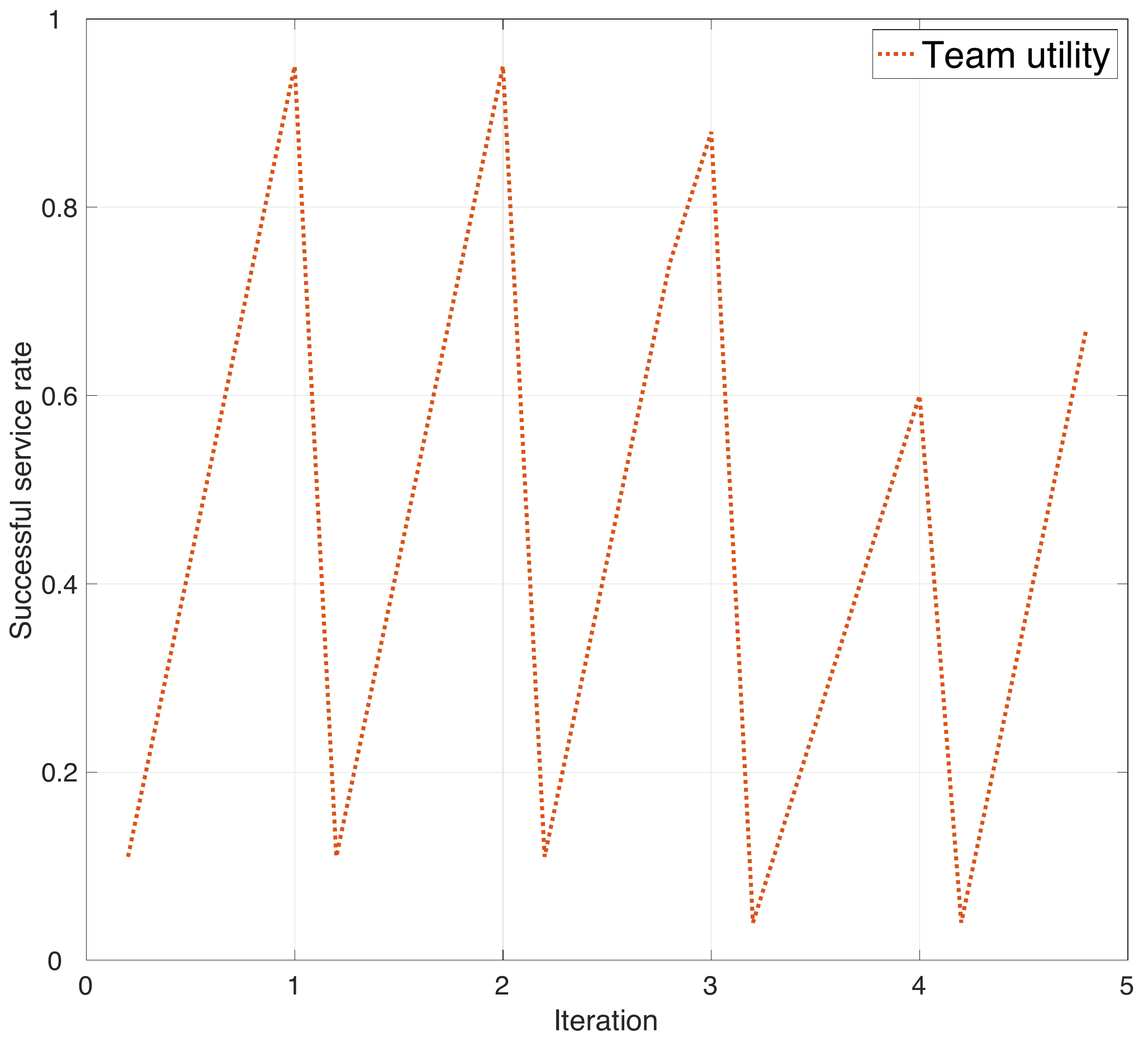}}
\caption{\footnotesize{Value decomposition within $5$ VD-RL iteration.}}
\label{VDRL2} %% label for entire figure
\vspace{-0.6cm}
\end{figure}
%\begin{figure}[t]
%%\setlength{\abovecaptionskip}{-5pt} 
%%\setlength{\belowcaptionskip}{-8pt} 
%  \begin{center}
%   \vspace{0cm}
%    \includegraphics[width=9.5cm]{VDRL2}
%    \vspace{-0.2cm}
%    \caption{\label{VDRL2} Value decomposition within $5$ VD-RL iteration. }
%  \end{center}\vspace{-0.9cm}
%\end{figure}

Fig. \ref{VDRL2} shows the value decomposition within $5$ VD-RL training iterations. In this figure, we can see that the value of $V\left(\boldsymbol{S}_0\right)$ increases at the first step within each VD-RL training iteration. This is because, function $V\left(\boldsymbol{S}_0\right)$ estimates the cumulative future team reward that the DBSs can get at the initial state. Thus, the DBSs update value function $V\left(\boldsymbol{S}_0\right)$ using the rewards that they achieved from serving the ground users at the last training iteration, once they return to the origin.  
%Note that, the $V\left(\boldsymbol{S}_0\right)$ keeps increasing after the sudden drop of final utility, as the adaptive moment estimation (Adam) algorithm \cite{kingma2014adam} which updates values with decaying average of past gradients is adopt in our simulation. 
Fig. \ref{VDRL2} also shows that the decomposed individual value $\tilde V_{\boldsymbol{\theta}_{c,2}}\left(\boldsymbol{S}_0\right)$ at DBS 2 is much larger than the individual value $\tilde V_{\boldsymbol{\theta}_{c,1}}\left(\boldsymbol{S}_0\right)$ at DBS 1.
 This is because DBS 2 selects the actions that are beneficial for the team utility. A large individual value $\tilde V_{\boldsymbol{\theta}_{c,2}}\left(\boldsymbol{S}_0\right)$ reinforces such actions to maintain a high team utility. However, DBS 1 does not select actions that improve the team utility.
Hence, the individual value $\tilde V_{\boldsymbol{\theta}_{c,1}}\left(\boldsymbol{S}_0\right)$ attributed to DBS 1 is small, which keeps DBS 1's strategy unchanged. DBS 1 will, thus, act ``greedily'' by selecting all possible actions with its current strategy, until it finds the optimal one.

\begin{figure}[t]
\setlength{\abovecaptionskip}{-5pt} 
\setlength{\belowcaptionskip}{-8pt} 
  \begin{center}
   \vspace{0cm}
    \includegraphics[width=9cm]{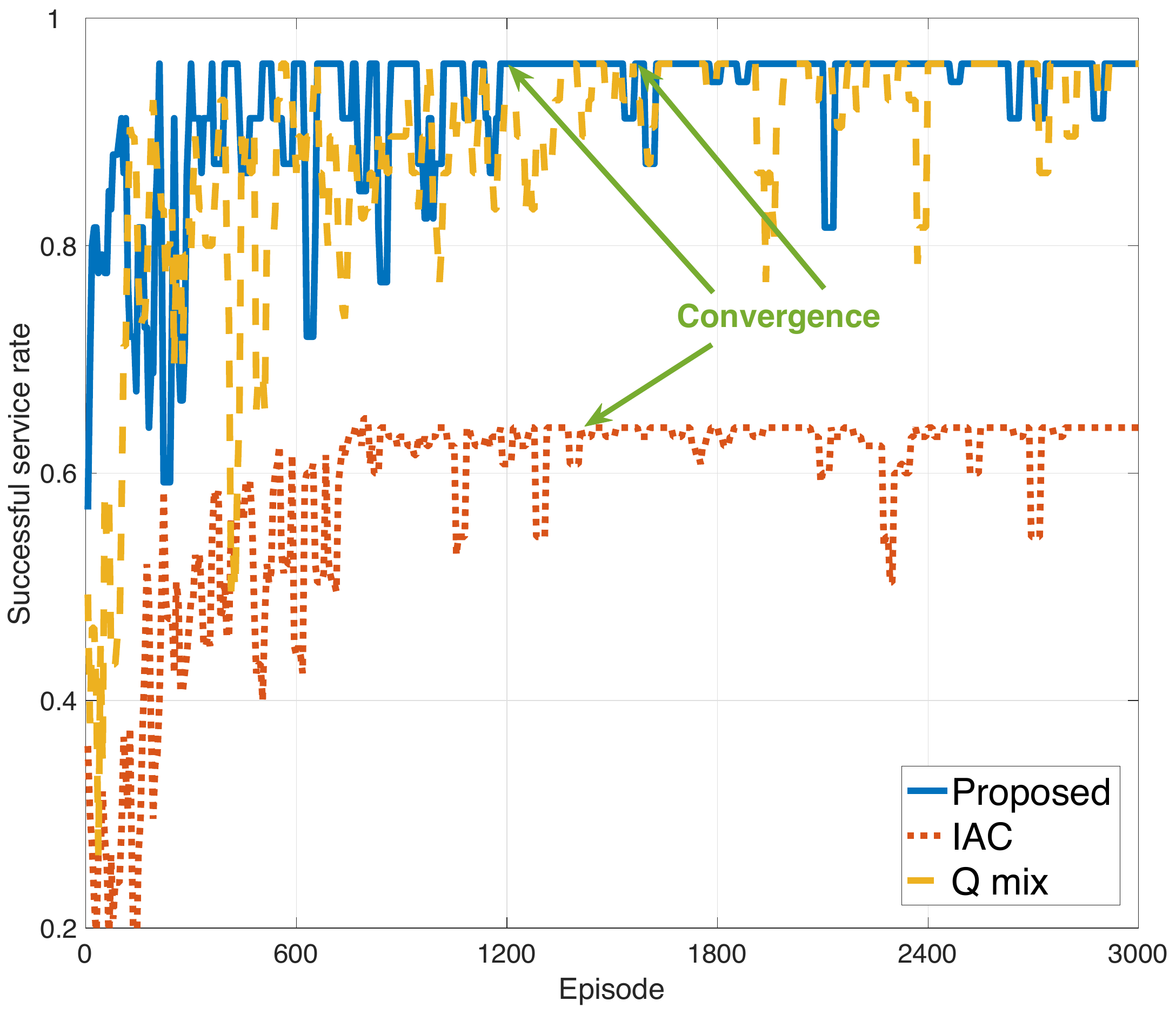}
    \vspace{-0.2cm}
    \caption{\label{VDRL3} Convergence of all considered multi-agent RL algorithms. }
  \end{center}\vspace{-0.9cm}
\end{figure}

In Fig. \ref{VDRL3}, we show the convergence of the proposed VD-RL algorithm. In this figure, we can see that the proposed VD-RL algorithm requires approximately $1,300$ iterations to reach convergence, which improves the convergence speed by up to $30.6\%$ compared to the Q mix algorithm. This stems from the fact that the neural network used to estimate the summation in (\ref{eq:vd}) remarkably increases the complexity of Q mix. Meanwhile, Fig. \ref{VDRL3} also shows that the proposed VD-RL algorithm yields a $53.2\%$ higher final utility than IAC. This is because the VD-RL algorithm can find a team optimal strategy to maximize the team reward. The IAC algorithm, however, find a strategy that maximize the DBSs' individual utilities. From Fig. \ref{VDRL3}, we also observe that the proposed VD-RL algorithm has a similar convergence speed to the IAC algorithm. This stems from the fact that, in the VD-RL algorithm, each DBS updates the policy and value functions using its own experience, and, thus, the proposed VD-RL algorithm has a low time complexity that is comparable to a single agent RL algorithm.

\begin{figure}
\centering
\subfigure[]{
\label{subfig:a} %% label for first subfigure
\includegraphics[width=5cm]{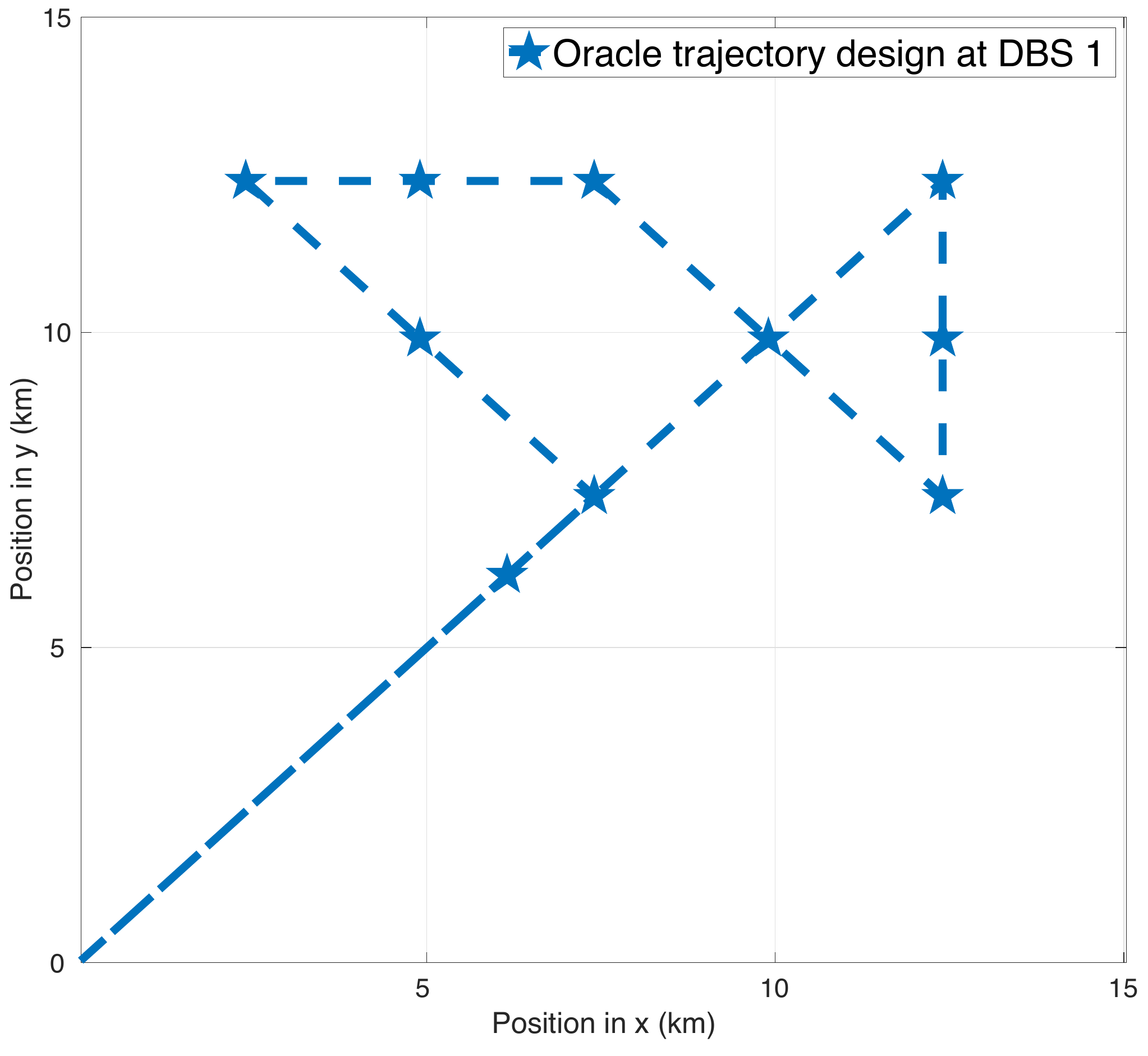}}
\subfigure[]{
\label{subfig:b} %% label for second subfigure
\includegraphics[width=5cm]{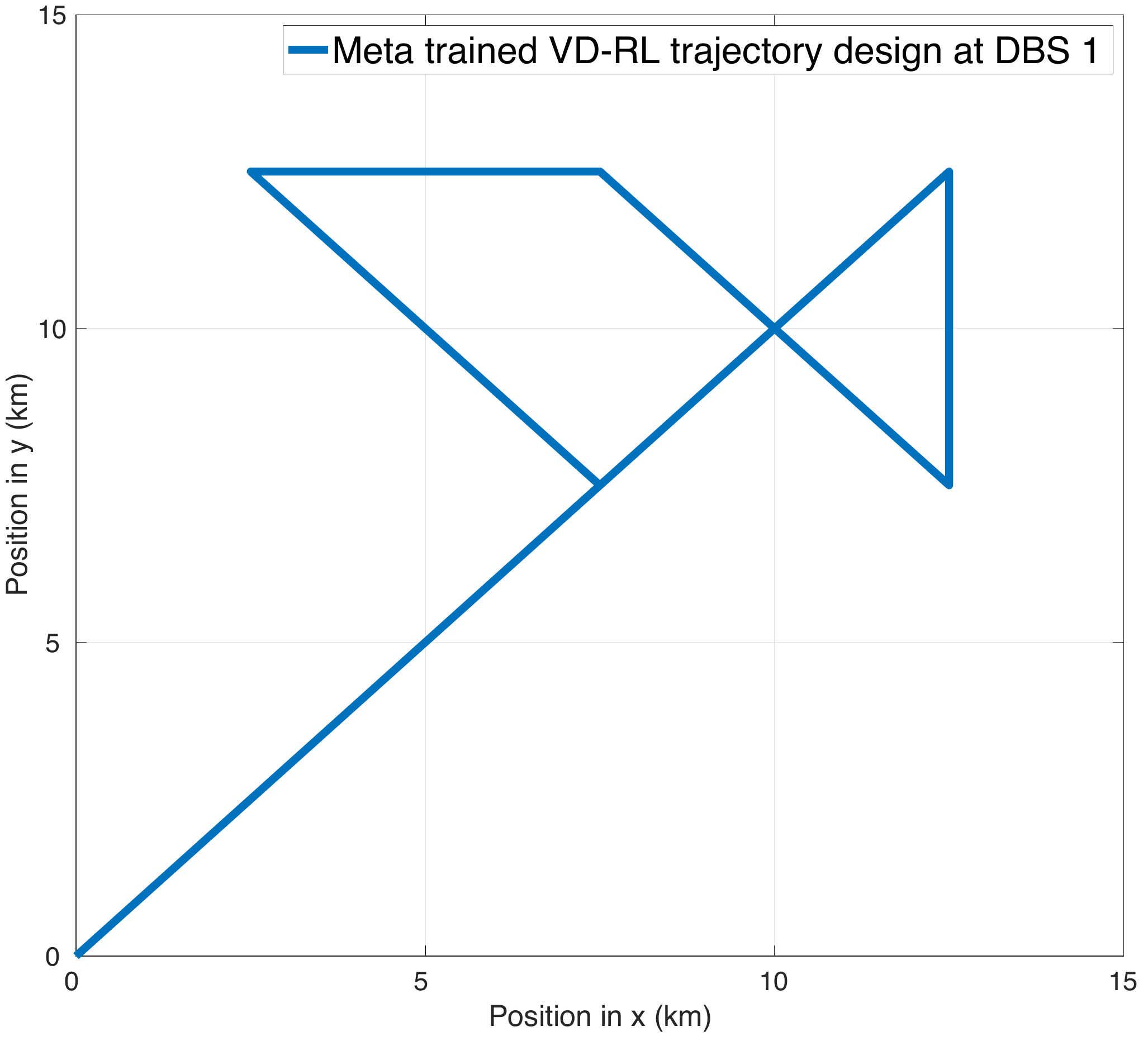}}
\subfigure[]{
\label{subfig:c} %% label for first subfigure
\includegraphics[width=5cm]{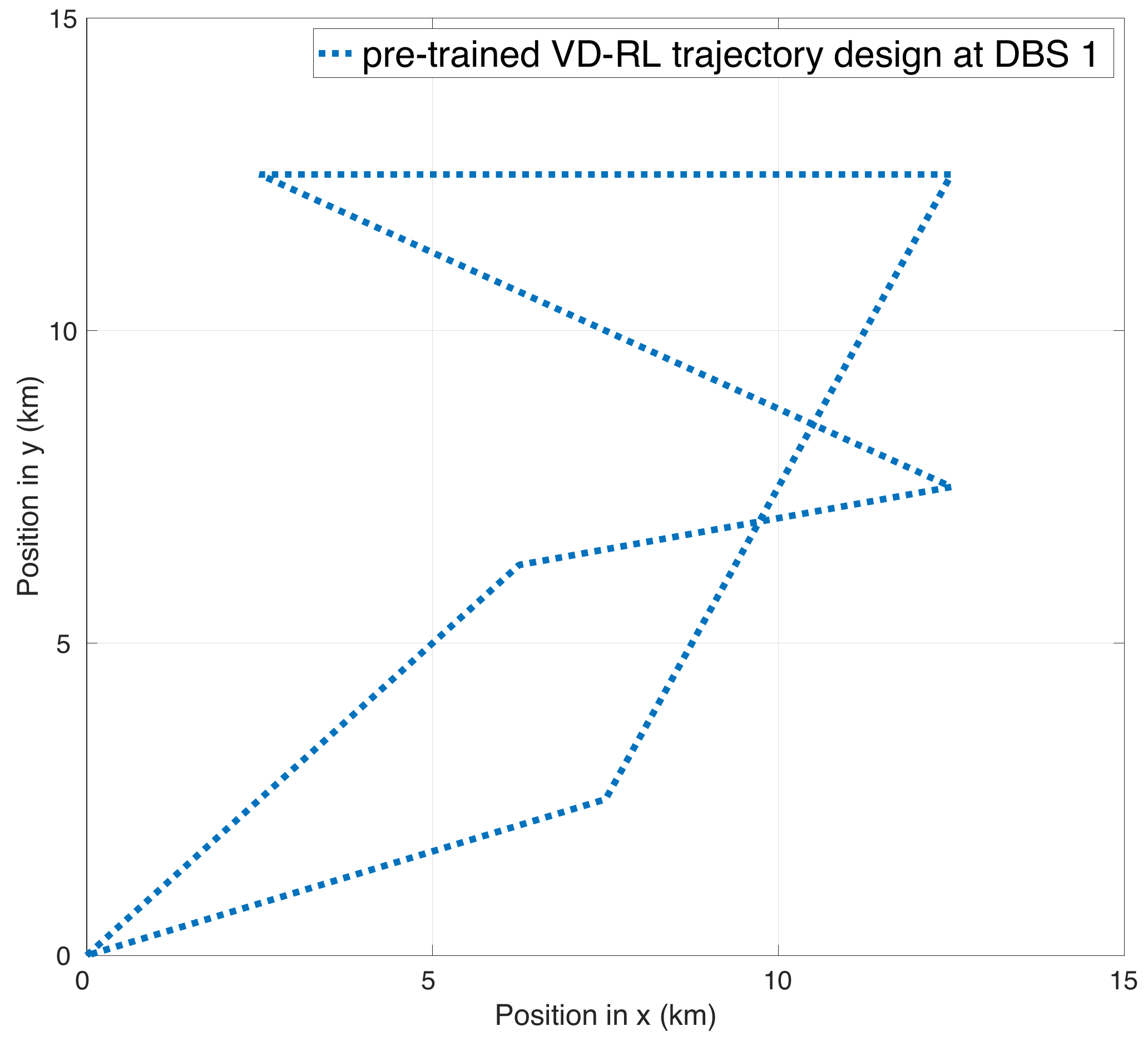}}
\subfigure[]{
\label{subfig:d} %% label for first subfigure
\includegraphics[width=5cm]{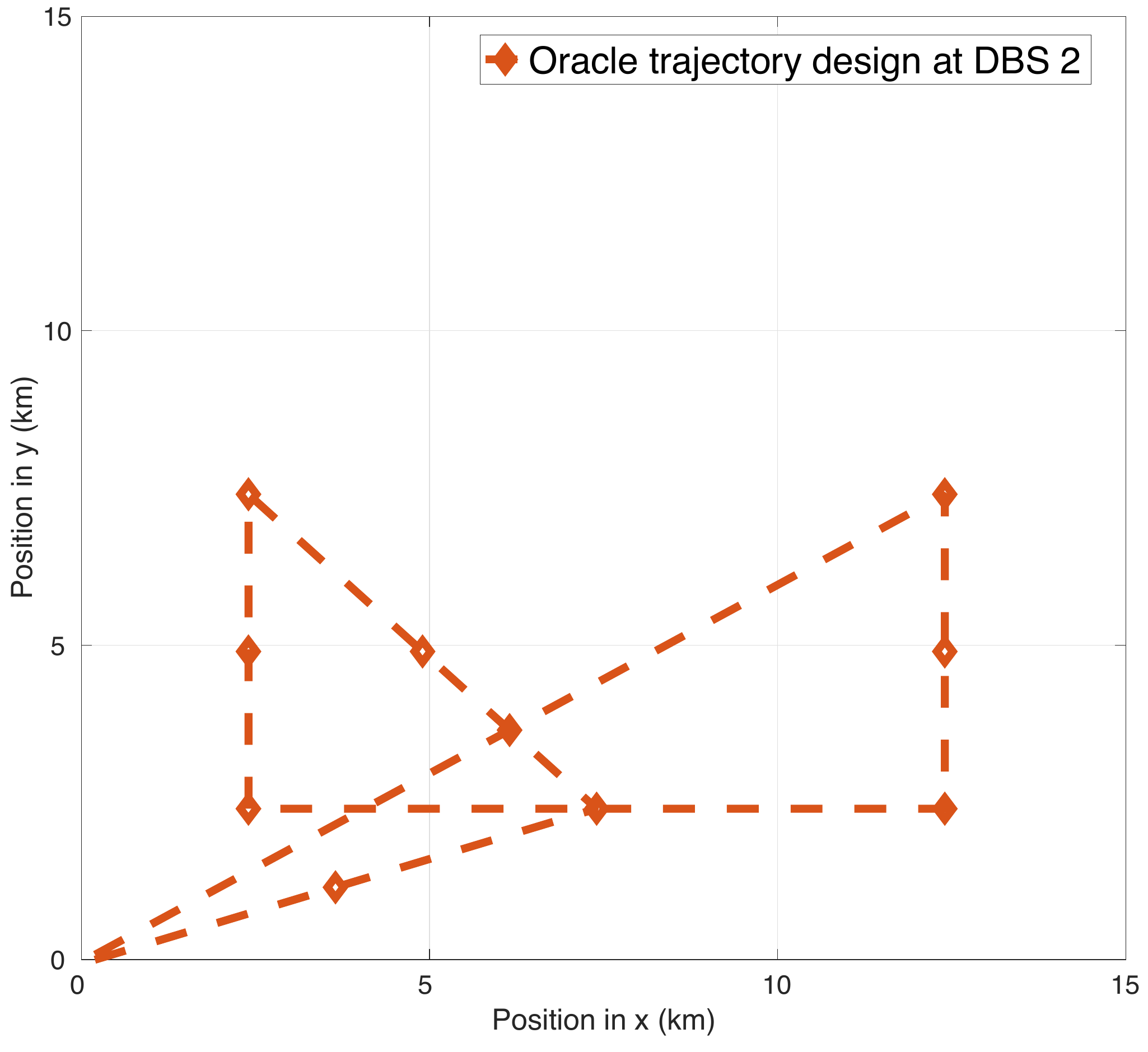}}
\subfigure[]{
\label{subfig:e} %% label for first subfigure
\includegraphics[width=5cm]{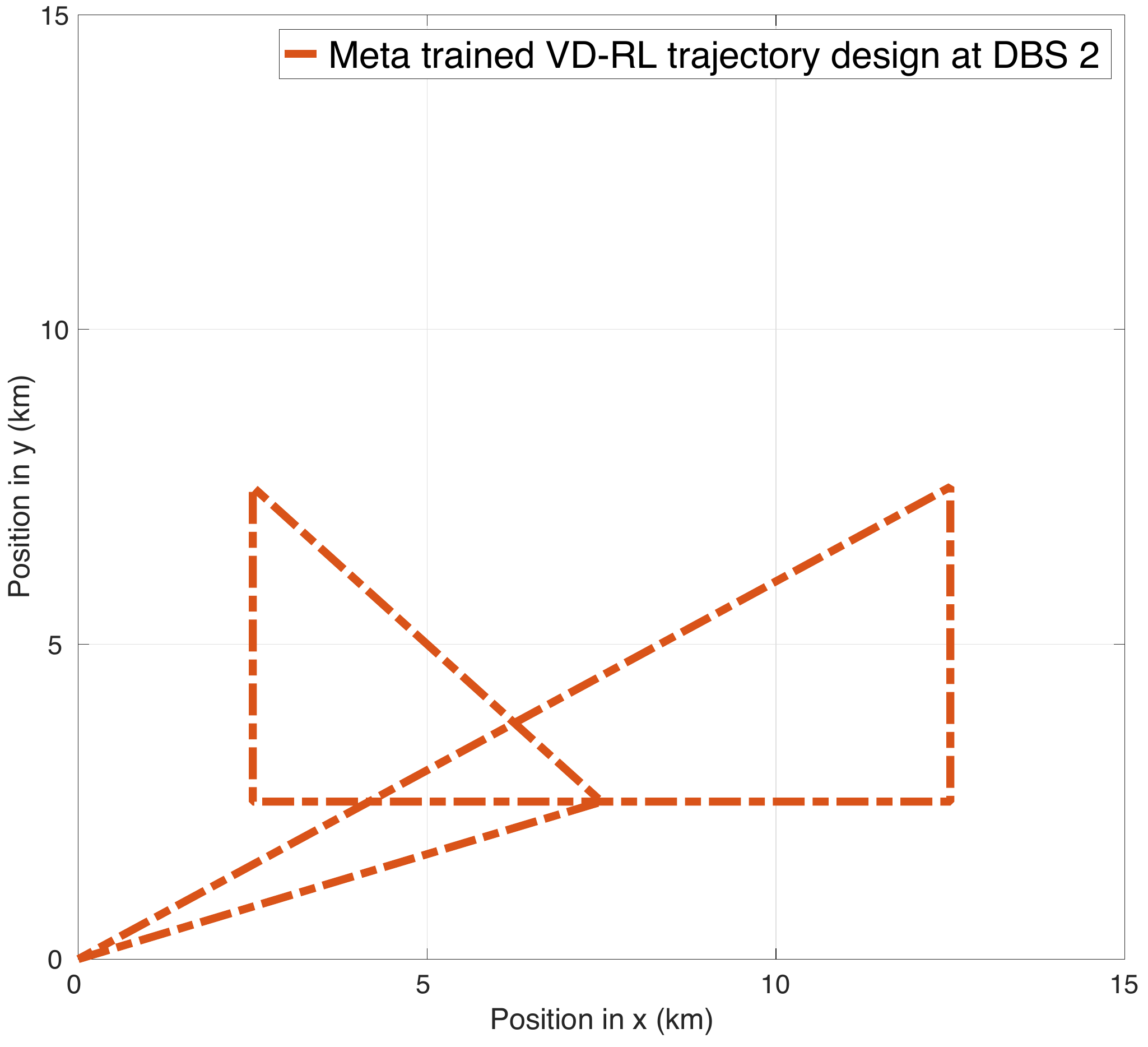}}
\subfigure[]{
\label{subfig:f} %% label for first subfigure
\includegraphics[width=5cm]{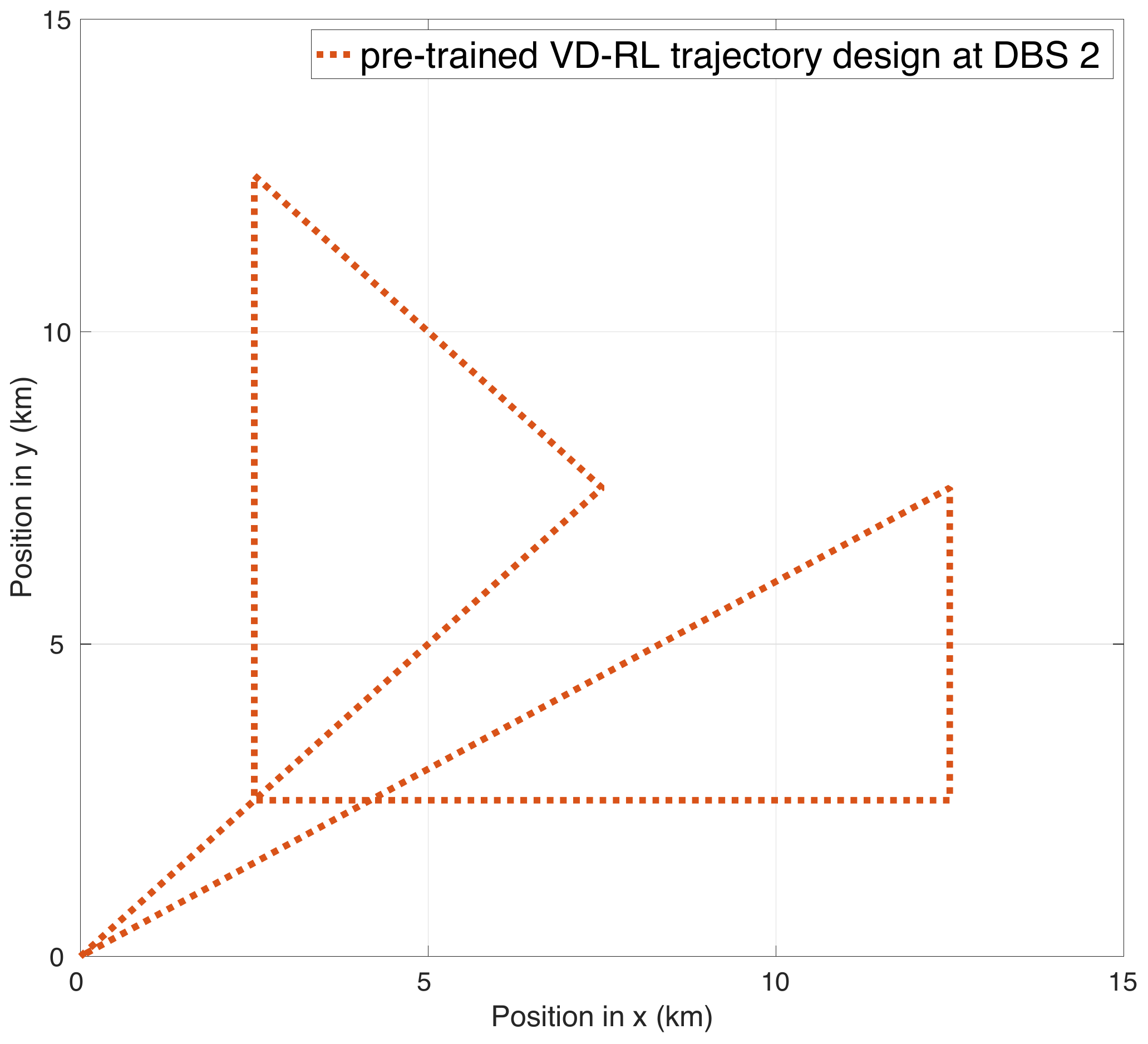}}
\caption{\footnotesize{Snapshots of trajectories resulting from all meta-trained and pre-trained algorithms in a sampled unseen environment.}}
\label{meta1} %% label for entire figure
%\setlength{\belowcaptionskip}{-15pt}
%\vspace{-0.9cm}
\end{figure}

Fig. \ref{meta1} shows a snapshot of the trajectories resulting from the proposed meta-trained VD-RL,  and the pre-trained VD-RL algorithm in an environment that was not experienced during the meta training procedure. Here, the trajectories in Figs. \ref{subfig:b}, \ref{subfig:c}, \ref{subfig:e}, and \ref{subfig:f} are selected, respectively, at the $100$-th iteration of both considered algorithms.
In this figure, we can see that, when faced with an unseen environment, the meta-trained VD-RL scheme can effectively find the optimal trajectories for the DBSs within a considerably small number of iterations. However, the trajectories resulting from the pre-trained VD-RL are still far from the optimal trajectories. Fig. \ref{meta1} also shows that the proposed meta-training method can adapt the DBSs in unseen environments much faster than the pre-trained VD-RL algorithm. This is because the meta-training method can find a policy and value function parameter initialization that is close to optimal policy and value functions at all possible user requests in $p\left(\mathcal{Z}\right)$. By using this meta-learning-based initialization, the meta-trained VD-RL can reach the optimal policy and value functions and find a team optimal strategy in an unseen environment within $p\left(\mathcal{Z}\right)$, using a small number of iterations.

\begin{figure}[t]
\setlength{\abovecaptionskip}{-5pt} 
\setlength{\belowcaptionskip}{-8pt} 
  \begin{center}
   \vspace{0cm}
    \includegraphics[width=9cm]{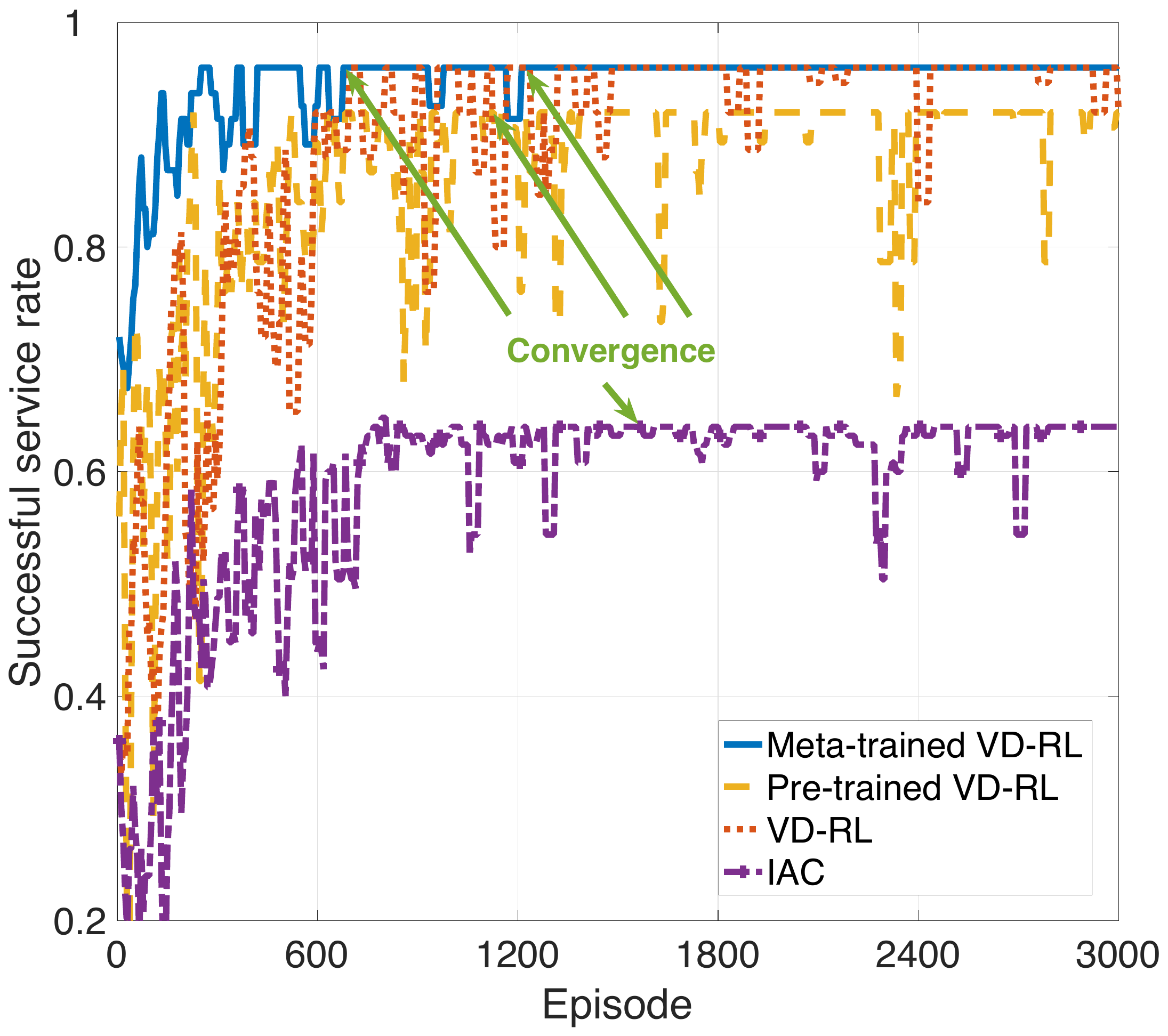}
    \vspace{-0.2cm}
    \caption{\label{meta2} Convergence of the meta-trained VD-RL. }
  \end{center}\vspace{-0.9cm}
\end{figure}

In Fig. \ref{meta2}, we show the convergence of the meta-trained VD-RL. From this figure, we can see that, with initial value functions $\tilde{V}_{\boldsymbol{\theta}^{*}_{c,n}}$ and policy functions $\boldsymbol{\pi}_{\boldsymbol{\theta}^{*}_{a,n}}$ provided by the meta training procedure in Algorithm \ref{alg:VD-MAMRL}, the meta-trained VD-RL converges at approximately the $700$-th iteration, which is $36.4\%$, and $53.8\%$ faster than the pre-trained VD-RL algorithm and the original VD-RL algorithm. Fig. \ref{meta2} also shows that the meta-trained VD-RL achieves a successful service rate that is equal to the one reached by the original VD-RL algorithm. The successful service rate achieved by the meta-training method is $9.2\%$ and $53.2\%$ higher than the one achieved by, respectively the pre-trained VD-RL algorithm and the IAC algorithm. 
This gain stems from the fact that the meta-training method trains a set of policy and value functions with proper estimation on various user access requests. By starting the policy updating procedure from such initialization, the DBSs can find a team optimal strategy within a small number of policy update steps. The pre-trained VD-RL, however, converges to a much lower successful service rate, since it could make the DBSs to start their policy updating procedure from the initializations that are far from the optimal policy and value functions.

\begin{figure}[t]
  \begin{center}
   \vspace{0cm}
    \includegraphics[width=9 cm]{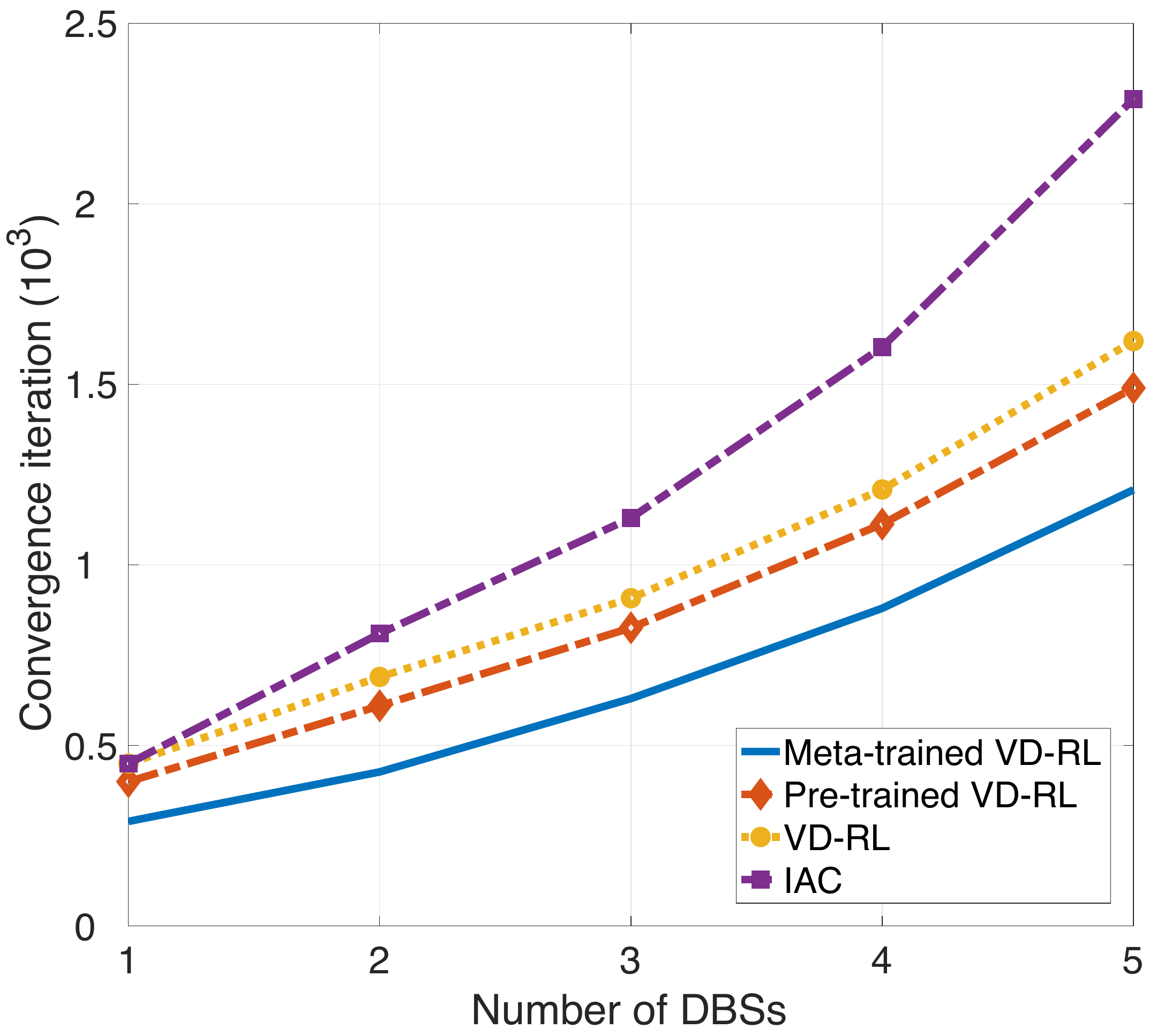}
    \vspace{-0.2cm}
    \caption{\label{meta3} Convergence of the meta-trained VD-RL algorithm as the number of DBSs varies. }
  \end{center}\vspace{-0.9cm}
\end{figure}

Fig. \ref{meta3} shows the convergence of the proposed meta-training method as the number of DBSs varies. Fig.  \ref{meta3} shows that, as the number of DBSs increases, the number of iterations needed for convergence increases because of the associated growth in the size of the action and state spaces.
From Fig.  \ref{meta3}, we observe that, as the number of DBSs increases, the number of iterations needed for convergence will grow at a higher speed. 
This is because the action and state spaces within the problem increase exponentially with the number of DBSs. This is consistent with the complexity analysis in Section III. A. Moreover, the number of iterations that the proposed VD-RL algorithm needs for convergence increases much slower compared to the traditional IAC algorithm. This is because the VD-RL algorithm reduces the dimensionality of the action and state spaces in the considered problem. In particular, the VD-RL algorithm enables each DBS to update its strategy using its own actions and states, thus simplifying the considered problem.
The convergence speed of the meta-trained VD-RL scheme decreases relatively slower compared to the original VD-RL algorithm, since the initializations of the policy and value function parameters in the meta-trained VD-RL algorithm have optimized performance across diverse user access request realizations. Moreover, the pre-trained VD-RL algorithm only slightly improves the VD-RL algorithm's convergence speed, as the initialization provided by the pre-trained VD-RL algorithm could be far from the optimal policy and value functions in the unseen environments.

\begin{figure}[t]
  \begin{center}
   \vspace{0cm}
    \includegraphics[width=9 cm]{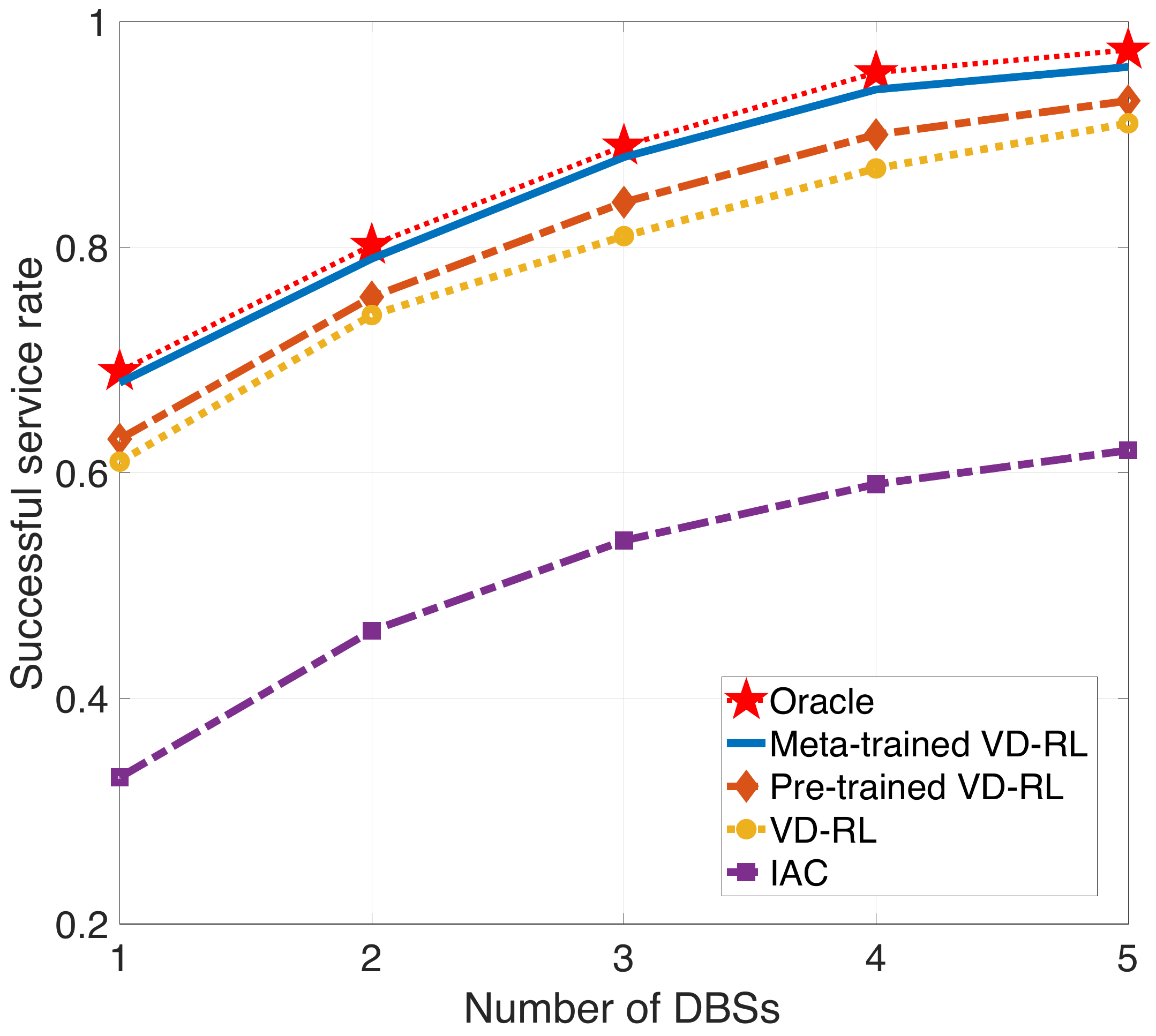}
    \vspace{-0.2cm}
    \caption{\label{meta4} Successful service rate as the number of DBSs varies. }
  \end{center}\vspace{-1cm}
\end{figure}
Fig. \ref{meta4} shows how the successful service rate achieved by the DBSs varies with the number of DBSs. In Fig. \ref{meta4}, we can see that, as the number of DBSs increases, the successful service rate of all the DBSs increases because having more DBSs can lead to a better coverage to the ground users. Fig. \ref{meta4} also shows that the successful service rate increases at a slower speed when the number of deployed DBSs increases. 
This is because that the number of un-covered user request decreases when the number of deployed DBSs increases.
From Fig. \ref{meta4}, we can also see that the meta-trained VD-RL algorithm yields a better coverage to the ground users, in particular, it yields a successful service rate that is $5.6\%$ higher than the one resulting from the VD-RL algorithm, $9.2\%$ higher than the one resulting from the pre-trained VD-RL algorithm, and  $53.2\%$ higher than the one resulting from the IAC algorithm.
This is due to the fact that, by reducing the losses defined in (\ref{eq:closs}) and (\ref{eq:ploss}) across various environments, the proposed meta-training method gets a set of policy and value functions that are close to the optimal policy and value functions at all possible user requests in $p\left(\mathcal{Z}\right)$. This guarantees a faster convergence to a team optimal strategy in an unseen environment. Note that,  although the meta-training method only yields small improvements to the successful service rate, it can converge much faster than the VD-RL algorithm.

\vspace{-0.2cm}
\section{Conclusion}
%\vspace{-0.2cm}
{
In this paper, we have studied the problem of trajectory design for a group of DBSs in unpredictable, dynamic environments. In the considered system, the DBSs cooperatively fly around the considered environment to provide on demand uplink communication service to ground users. We have formulated the studied problem in an optimization setting and have proposed a VD-RL algorithm to solve this problem. The proposed VD-RL algorithm makes the DBSs independently update their individual strategies toward the maximal team utility of the DBSs, by sharing only its utility and value to other DBSs.
To improve the convergence speed of the VD-RL algorithm in unseen environments, we have also proposed a meta-training method to optimize the initializations in a VD-RL solution.
Simulation results show that the proposed VD-RL algorithm with meta training mechanism outperforms the traditional MARL algorithms. %{{Future work could include the multi-antenna scenarios with the consideration of precoding technique.}}
}
 \
 \vspace{-0.7cm}
 \section*{Appendix}
{

 \subsection{Proof of Proposition \ref{proposition2}}
 \begin{proof} 
 To prove the convergence of the proposed VD-RL algorithm, we only need to prove that the proposed VD-RL algorithm satisfies the following conditions in \cite{sutton2000policy}:
 {\begin{enumerate}%[label=(\roman*)]
	\item Value function $V\left(\boldsymbol{S}_k\right)$ converges to a local minimum value of $ {A}^2\left(\boldsymbol{a}_{k}, {\boldsymbol{S}_{k}}\right)$, that is
		   \begin{equation}\label{eq:pf00}
\begin{split}
%&\nabla_{\boldsymbol{\theta}_{c}} \sum^K_{k=1}A^2\left(\boldsymbol{a}_{k}, {\boldsymbol{S}_{k}}\right) =0.
2\sum^K_{k=1}A\left(\boldsymbol{a}_{k}, {\boldsymbol{S}_{k}}\right)\nabla_{\boldsymbol{\theta}_{c}}A\left(\boldsymbol{a}_{k}, {\boldsymbol{S}_{k}}\right)=0.
\end{split}
\end{equation} 
	\item Policy function parameters are updated as in %$\boldsymbol{\theta}^{\left(i+1\right)}_{a}=\boldsymbol{\theta}_{a}^{\left(i\right)}+\alpha^{\left(i\right)}_a\sum^K_{k=1} A\left(\boldsymbol{a}_{k}, {\boldsymbol{S}_{k}}\right)\nabla_{\boldsymbol{\theta}_{a}}\log \pi_{\boldsymbol{\theta}_{a}}\left(\boldsymbol{a}_{k}\left| {\boldsymbol{S}_{k}} \right.\right)$.
	   \begin{equation}\label{eq:pf0}
\begin{split}
%&\nabla_{\boldsymbol{\theta}_{c}} \sum^K_{k=1}A^2\left(\boldsymbol{a}_{k}, {\boldsymbol{S}_{k}}\right) =0.
\boldsymbol{\theta}^{\left(i+1\right)}_{a}=\boldsymbol{\theta}_{a}^{\left(i\right)}+\alpha^{\left(i\right)}_a\sum^K_{k=1} A\left(\boldsymbol{a}_{k}, {\boldsymbol{S}_{k}}\right)\nabla_{\boldsymbol{\theta}_{a}}\log \pi_{\boldsymbol{\theta}_{a}}\left(\boldsymbol{a}_{k}\left| {\boldsymbol{S}_{k}} \right.\right).
\end{split}
\end{equation} 
	\item $\mathop {\max }\limits_{\boldsymbol{a}_{n,k}\in{\mathcal{C}\cup\left\{O\right\}},\boldsymbol{s}_{n,k}\in{\mathcal{S}} }\left|\frac{\partial \pi_{\boldsymbol{\theta}_{a}}\left(\boldsymbol{a}_{k}\left| {\boldsymbol{S}_{k}} \right.\right)}{\partial \theta_{i} \partial \theta_{j}}\right|<\infty$, for any elements $\theta_{i}$ and $\theta_{j}$ in the policy function parameter vector $\boldsymbol{\theta}_{a}$.
	\item $\lim_{i \to \infty}\alpha^{\left(i\right)}_c=0$, $\sum^{\infty}_{i=1}\alpha^{\left(i\right)}_c=\infty$, $\lim_{i \to \infty}\alpha^{\left(i\right)}_a=0$, and $\sum^{\infty}_{i=1}\alpha^{\left(i\right)}_a=\infty$. 
\end{enumerate}}
 %Note that, here, ${\boldsymbol{\theta}_{a}}=\left[\boldsymbol{\theta}_{a,1}, \boldsymbol{\theta}_{a,2}, \ldots, \boldsymbol{\theta}_{a,N}\right]$ is the vector of policy function parameters.
 
Next, we prove that the proposed VD-RL algorithm satisfies condition 1). When value function $\tilde{V}_{\boldsymbol{\theta}_{c,n}}\left(\boldsymbol{s}_{n,k}\right)$ converges to a local minimum, such that
   \begin{equation}\label{eq:pf1}
\begin{split}
%&\nabla_{\boldsymbol{\theta}_{c}} \sum^K_{k=1}A^2\left(\boldsymbol{a}_{k}, {\boldsymbol{S}_{k}}\right) =0.
&2 A\left(\boldsymbol{a}_{k}, {\boldsymbol{S}_{k}}\right)\nabla_{\boldsymbol{\theta}_{c,n}} \left(\gamma  \tilde{V}_{\boldsymbol{\theta}_{c,n}}\left(\boldsymbol{s}_{n,k+1}\right)-\tilde{V}_{\boldsymbol{\theta}_{c,n}}\left(\boldsymbol{s}_{n,k}\right) \right)\\
&= 2 A\left(\boldsymbol{a}_{k}, {\boldsymbol{S}_{k}}\right)\nabla_{\boldsymbol{\theta}_{c,n}} \tilde {A}_n\left({a}_{n,k}, {\boldsymbol{s}_{n,k}}\right)=0.
\end{split}
\end{equation} 
with the assumption $A\left(\boldsymbol{a}_{k}, {\boldsymbol{S}_{k}}\right)\neq0$, we have $\nabla_{\boldsymbol{\theta}_{c,n}} \tilde {A}_n\left({a}_{n,k}, {\boldsymbol{s}_{n,k}}\right)=0$. That is, we have $2\sum^K_{k=1}\tilde {A}_n\left({a}_{n,k}, {\boldsymbol{s}_{n,k}}\right)\nabla_{\boldsymbol{\theta}_{c,n}}\tilde {A}_n\left({a}_{n,k}, {\boldsymbol{s}_{n,k}}\right)=0$.
% Give the value decomposition assumptions in (\ref{eq:vd}) and (\ref{eq:vd2}), we have
%     \begin{equation}\label{eq:pf2}
% % \setlength{\abovedisplayskip}{3 pt}
%%\setlength{\belowdisplayskip}{3 pt}
%\begin{split}
%&2 \left[\sum^K_{k=1} A\left(\boldsymbol{a}_{k}, {\boldsymbol{S}_{k}}\right)\nabla_{\boldsymbol{\theta}_{c,n}} \tilde {A}_n\left({a}_{n,k}, {\boldsymbol{s}_{n,k}}\right)\right]_{n\in\mathcal{N}}\\
%&= 2 \sum^K_{k=1} A\left(\boldsymbol{a}_{k}, {\boldsymbol{S}_{k}}\right)\nabla_{\boldsymbol{\theta}_c} \sum_{n\in\mathcal{N}}\tilde {A}_n\left({a}_{n,k}, {\boldsymbol{s}_{n,k}}\right).\\
%&= \nabla_{\boldsymbol{\theta}_c} \sum^K_{k=1}A^2\left(\boldsymbol{a}_{k}, {\boldsymbol{S}_{k}}\right)=0,
%\end{split}
%\end{equation} 
% Thus, when the individual value functions converge, value function $V\left(\boldsymbol{S}_k\right)$ also converges with $\nabla_{\boldsymbol{\theta}_{c}} \sum^K_{k=1}A^2\left(\boldsymbol{a}_{k}, {\boldsymbol{S}_{k}}\right) = 0$. 
 This means that, the VD-RL algorithm satisfies condition 1).
 
 Next, we can see that the VD-RL's distributed update on policy function parameters satisfies condition 2), as each DBS $n$ updates its policy function parameters in the form of $\boldsymbol{\theta}^{\left(i+1\right)}_{a,n}=\boldsymbol{\theta}_{a,n}^{\left(i\right)}+\alpha^{\left(i\right)}_a\sum^K_{k=1}\tilde {A}_n\left({a}_{n,k}, {\boldsymbol{s}_{n,k}}\right)\nabla_{\boldsymbol{\theta}_{a,n}}\log \pi_{\boldsymbol{\theta}_{a,n}}\left(a_{n,k}\left| {\boldsymbol{s}_{n,k}} \right.\right)$. Thus, the distributed update on policy parameters {$\boldsymbol{\theta}_{a,n}$} satisfies condition 2). Condition 3) can be satisfied by properly setting neural network of the policy functions in the proposed VD-RL algorithm (e.g. setting activation function). Meanwhile, condition 4) can be satisfied by adjusting the step sizes of the policy and value functions. In consequence, the proposed algorithm implemented at each DBS $n$ satisfies all conditions from 1) to 4). In other words, each DBS $n$ is guaranteed to converge to an optimal strategy $\boldsymbol{\pi}^*_{n}$ that yields a local maximal team utility $\overline G\left(\boldsymbol{\pi}^*_n, \boldsymbol{\pi}_{-n}\right)$, with DBSs $n'\in\mathcal{N}\setminus n$ following strategies in $\boldsymbol{\pi}_{-n}=\left[\boldsymbol{\pi}_{n}\right]_{n'\in\mathcal{N}\setminus n}$. In summery, by updating policies at each DBS in the system, the proposed VD-RL algorithm solves problem (\ref{opt}) step by step in the form of
 \addtocounter{equation}{0}
\setlength{\belowdisplayskip}{-2 pt}
\begin{equation}\label{opttrans}
%\begin{split}
 \max_{\boldsymbol{\pi}_1} \max_{\boldsymbol{\pi}_2}\ldots \max_{\boldsymbol{\pi}_N} \sum_{\boldsymbol{\xi}\in\mathcal{E}}G\left(\boldsymbol{\xi}\right)\prod\limits^{K}_{k=1}\pi_1\left(\xi\left| {{\xi}_{1,k}},\tau_{1,k} \right.\right) \prod\limits^{K}_{k=1}\pi_2\left(\xi\left| {{\xi}_{2,k}},\tau_{2,k} \right.\right) \ldots \prod\limits^{K}_{k=1}\pi_N\left(\xi\left| {{\xi}_{N,k}},\tau_{N,k} \right.\right)    %\overline G\left(\boldsymbol{\pi}\right),
%\end{split}
\end{equation}
\vspace{-0.8cm}
\begin{align}\label{c1}
\setlength{\abovedisplayskip}{-5 pt}
\setlength{\belowdisplayskip}{-2 pt}
&\;\;\;\;\rm{s.\;t.}\scalebox{1}{$\;\;\;\;  \sum_{\boldsymbol{\xi}\in\mathcal{E}}\prod\limits^{N}_{n=1}\prod\limits^{K}_{k=1}\pi_n\left(\xi\left| {{\xi}_{n,k}},\tau_{n,k} \right.\right)=1, $} \tag{\theequation a}\\
&\;\;\;\;\;\;\;\;\scalebox{1}{$\;\;\;\;\;\sum_{{a}_{n,k}\in\mathcal{C} \cup \left\{O\right\}}\pi_n\left(\xi\left| {{\xi}_{n,k}},\tau_{n,k} \right.\right)=1, \forall n\in \mathcal{N},  \boldsymbol{\xi}\in \mathcal {E}, k \in \mathcal{K},$} \tag{\theequation b}\\
&\;\;\;\;\;\;\;\;\scalebox{1}{$\;\;\;\;\;0\le \pi_n\left(\xi\left| {{\xi}_{n,k}},\tau_{n,k} \right.\right)\le 1, \forall n\in \mathcal{N}, \boldsymbol{\xi}\in \mathcal {E}, k \in \mathcal{K},$} \tag{\theequation c}
\end{align}
As the DBSs' strategies are independent, the local optimal strategy at each DBS $n$ constructs a local optimal strategy of the non-convex problem (\ref{opt}). 
This completes the proof.

 \end{proof}

 }
\vspace{-0.7cm}

\bibliographystyle{IEEEbib}
\def\baselinestretch{1.35}
\bibliography{references}
\end{document}